\documentclass[preprint]{l4dc2026}

\title[Spectral Clipping for Learning Stable Linear and Latent-Linear Dynamical Systems]{On the Surprising Effectiveness of Spectral Clipping in Learning Stable Linear and Latent-Linear Dynamical Systems}
\usepackage{times}

\usepackage[utf8]{inputenc} 
\usepackage[T1]{fontenc}    
\usepackage{hyperref}       
\usepackage{url}            
\usepackage{booktabs}       
\usepackage{amsfonts}       
\usepackage{nicefrac}       
\usepackage{microtype}      
\usepackage{xcolor}         
\usepackage{array}
\usepackage{subcaption}
\usepackage{multirow}
\usepackage{graphicx}
\usepackage{amsmath}
\usepackage{bm}
\usepackage{wrapfig}
\usepackage{titlesec}
\usepackage{mathrsfs}
\usepackage{mathtools}
\usepackage{xcolor}
\usepackage{booktabs}

\newcommand{\LS}[1]{\textcolor{gray}{#1}}

\coltauthor{%
 \Name{Hanyao Guo}\thanks{Equal contribution, name ordered alphabetically.}  \Email{hguo323@gatech.edu}\\
 \Name{Yunhai Han}\footnotemark[1] \Email{yhan389@gatech.edu}\\
 \Name{Harish Ravichandar} \Email{harish.ravichandar@gatech.edu}\\
 \addr Georgia Institute of Technology,
  Atlanta, GA 30332%
 \vspace{-15pt}
}

\begin{document}

\maketitle

\vspace{-15pt}
\begin{abstract}
When learning stable linear dynamical systems from data, three important properties are desirable: i) predictive accuracy, ii) verifiable stability, and iii) computational efficiency. Unconstrained minimization of prediction errors leads to high accuracy and efficiency but cannot guarantee stability. Existing methods to enforce stability often preserve accuracy, but do so only at the cost of increased computation. In this work, we investigate if a seemingly-naive procedure can \textit{simultaneously} offer all three desiderata. Specifically, we consider a post-hoc procedure in which we surgically manipulate the spectrum of the linear system after it was learned using unconstrained least squares. We call this approach \textit{spectral clipping} (SC) as it involves eigen decomposition and subsequent reconstruction of the system matrix after any eigenvalues whose magnitude exceeds one have been clipped to one (without altering the eigenvectors).  We also show that SC can be readily combined with Koopman operators to learn nonlinear dynamical systems that can generate stable predictions of nonlinear phenomena, such as those underlying complex dexterous manipulation skills involving multi-fingered robotic hands. Through comprehensive experiments involving two different applications and publicly available benchmark datasets, we show that this simple technique can efficiently learn highly-accurate predictive dynamics that are provably-stable. Notably, we find that SC can match or outperform strong baselines while being \textit{orders-of-magnitude} faster. Finally, we find that SC can learn stable robot policies even when the training data includes unsuccessful or truncated demonstrations. Our code and datasets can be found at \href{https://github.com/GT-STAR-Lab/spec_clip}{https://github.com/GT-STAR-Lab/spec\_clip}. 
\end{abstract}


\section{Introduction}

In spite of the growing dominance of deep neural networks, modeling and learning of linear dynamical systems (LDS) remain relevant due to the availability of analytical solutions, significantly lower computational burden, and ease of analyses and inspection. LDS continue to find use in diverse applications such as computer vision~\citep{chan2005probabilistic, boots2007constraint, chien2022gpu, wang2024localized} and time series prediction~\citep{liu2024koopa, takeishi2017learning}. Even if the underlying dynamics are nonlinear, researchers are increasingly utilizing tools from linear systems theory by leveraging the Koopman Operator theory~\citep{Koopman1931Koopman, williams2015data, mezic2020koopman} to approximate nonlinear systems with linear systems in a higher-dimensional lifted state space. These techniques further increase the relevance and importance of learning LDS in complex applications, such as 
legged robotics~\citep{kim2024learning, li2024continual}, 
robotic manipulation~\citep{Bruder2021Soft,han2023KODex, han2024learning, chen2024korol}, atmospheric sciences~\citep{navarra2021estimation, millard2024deep, zheng2022hybrid}, and fluid dynamics~\citep{mezic2013analysis, schmid2022dynamic}. 

When time-series data associated with system evolution is available, learning LDS can be considered as a form of self-supervised learning and can be solved efficiently by minimizing a regression objective~\citep{ljung1998system}. However, such unconstrained optimization methods tend to ignore important properties of the underlying systems, such as stability. Consequently, even if the underlying system is known to be stable, the stability of learned system could be highly sensitive to the training data distribution, potentially leading to unstable solutions and catastrophic long-term predictions.


A number of methods have been developed to enforce stability of learned LDS and reduce sensitivity to training data~\citep{lacy2003subspace_conf, lacy2003subspace, boots2007constraint, huang2016learning, gillis2020note, mamakoukas2020memory, mamakoukas2023learning}. These methods take one of three approaches: a) minimizing the regression objective under stability constraints~\citep{lacy2003subspace_conf, lacy2003subspace}, b) iterating between solving 
an unconstrained problem and refining the solution under stability constraints~\citep{boots2007constraint, huang2016learning}, and c) iterating between characterizing the stable matrix and optimizing the characterizations with the regression objective~\citep{gillis2020note, mamakoukas2020memory, mamakoukas2023learning}. 
In general, these methods inevitably introduce a considerable computational burden, potentially countering one of the primary advantages of linear techniques. 

In this work, we investigate a simple yet surprisingly effective method, \textit{Spectral Clipping} (\textit{SC}), to learn stable LDS from data with little extra computation. SC involves a straightforward implementation: 1) compute the system matrix via unconstrained least-squares, 2) perform eigen decomposition of the system matrix and extract all eigenvalues, 3) \textit{clip} the magnitude of eigenvalues only when larger than one to one, and 4) reconstruct the system evolution matrix using the clipped eigenvalues and the \textit{unchanged} eigenvectors. 

In contrast to existing approaches discussed above, \textit{SC} offers \textit{stability-by-construction} without requiring constrained optimization or iterative updates.  
\textit{SC} can efficiently scale to \textit{high-dimensional} systems (e.g., image prediction). Further, seminal works on the Koopman operator~\citep{rowley2009spectral, mezic2013analysis, mezic2020koopman} reveal fundamental connections between the spectrum of the Koopman matrix and desirable properties in learned \textit{nonlinear systems}. We show that \textit{SC} can be combined with Koopman operators to enhance the predictive accuracy and stability of learned nonlinear dynamical systems (e.g., robotic systems).

While it can be readily seen that SC ensures stability by directly manipulating the eigenvalues, it might seem puzzling that it does not hurt the model's predictive accuracy.  
To explain SC's surprising effectiveness, we provide practical intuitions and theoretical insights (e.g., see Fig.~\ref{fig:Geometric_interpretation} for a geometric interpretation of \textit{SC}). Indeed, prior works have identified several linear techniques that defy expectations: controlling the singular values of linear layers improves training stability and generalization~\citep{boroojeny2024spectrum,sedghi2018singular, senderovich2022towards}; retraining only the last linear layer improves robustness to spurious correlations~\citep{kirichenko2022last}; and Low-Rank Adaptation (LoRA) dramatically improves efficiency~\citep{hu2022lora}. 
We conducted comprehensive experiments involving benchmark datasets across different domains (two video prediction datasets~\citep{chan2005probabilistic, hadji2018new}, and a robot model learning dataset~\citep{gaz2019dynamic}) and strong baselines (CG~\citep{boots2007constraint}, WLS~\citep{huang2016learning}, and SOC~\citep{mamakoukas2020memory}). Our results reveal that \textit{SC} consistently achieves comparable (if not better) predictive accuracy, and generates stable long-horizon predictions. Notably, \textit{SC} is able to achieve this level of performance while being \textit{orders-of-magnitude} faster.

Our experiments also show that \textit{SC} can be readily combined with Koopman operators to learn \textit{nonlinear} systems capable of stable and consistent predictions in mere \textit{seconds} (e.g., behavioral dynamics of dexterous manipulation skills for a multi-fingered robotic hand~\citep{han2023KODex}). Importantly, we show that \textit{SC} enables the learned system to generate robust and stable predictions even when the training data includes unsuccessful or truncated demonstrations (commonly characteristics of real-world robot learning datasets~\citep{qin2022dexmv, shaw2023videodex, bahl2022human}).
In contrast, we find unconstrained learning yields unstable LDS leading to erratic and ineffective robot behaviors in the presence of such challenging characteristics.


In summary, our core contributions include: i) a simple post-hoc procedure to ensure stability of learned linear dynamical systems, ii) evidence for its surprising ability to simultaneously offer reliable predictions, high computational efficiency, and stability guarantees, iii) practical and theoretical groundings of its underlying functionality tunable trade-off between expressivity and stability, iv) an extension leveraging Koopman operators to benefit nonlinear systems, and v) enabling robustness against unsuccessful and truncated demonstrations.


\section{Background: Learning Stable Linear Dynamical Systems}
\label{linear_stable}

Consider a stable but unknown discrete-time autonomous linear dynamical system $\mathrm{x}_{t+1} = A \mathrm{x}_{t}$,
where $\mathrm{x}_k \in \mathcal{X} \subseteq \mathbb{R}^n$ is the system state at time step $k \in \mathbb{N}$, and $A \in \mathcal{S}_{A}$ is the unknown system matrix, and $\mathcal{S}_{A} \subset \mathbb{R}^{n \times n}$ is the subspace that contains Schur stable matrices: $\mathcal{S}_{A} \triangleq \{A \in \mathbb{R}^{n \times n}| \max \{|\lambda_i(A)|\}^n_{i=1} < 1\}$, where $\{\lambda_i(A)\}^n_{i=1}$ are the eigenvalues of the matrix $A$, and $\vert \lambda \vert $ is the $|\lambda| \coloneqq \sqrt{\text{Re}(\lambda)^2 + \text{Im}(\lambda)^2}$.
Suppose we have a dataset $D_A = \{[\mathrm{x}_1^{(j)}, \mathrm{x}_2^{(j)}, \cdots, \mathrm{x}_T^{(j)}]\}_{j=1}^{N_A}$ containing $N_A$ trajectories of the stable LDS. We can now define two data matrices: $X^{(j)} = [\mathrm{x}_1^{(j)}, \mathrm{x}_2^{(j)}, \cdots, \mathrm{x}_{T-1}^{(j)}] \in \mathbb{R}^{n \times (T-1)}$ and $Y^{(j)} = [\mathrm{x}_2^{(j)}, \mathrm{x}_3^{(j)}, \cdots, \mathrm{x}_T^{(j)}] \in \mathbb{R}^{n \times (T-1)}$.

\noindent\textbf{Unconstrained Learning}: In its simplest form, the problem of learning an LDS from $D$ reduces to solving a standard least-squares problem.~\citep{montgomery2021introduction, seber2012linear}:
\setlength{\abovedisplayskip}{2pt}
\setlength{\belowdisplayskip}{2pt}
\begin{equation}
\label{eqn:data_driven}
\hat{A}_{LS} = \arg\min_{\Theta} \sum_{j=1}^{N_A} || Y^{(j)} - \Theta X^{(j)}||_F^2
\end{equation}
where $\Theta\in \mathbb{R}^{n \times n}$ are the learnable parameters and $||\cdot||_F$ is the matrix Frobenius norm. 
However, the optimization in Eq.~(\ref{eqn:data_driven}) does not impose any constraints on the matrix $A$, resulting in the learning of potentially unstable systems despite our knowledge that the original system is stable.


\noindent\textbf{Stability-Constrained Learning}: The prevailing strategy to ensure that the learned LDS is stable when learning from a dataset $D$ involves rewriting the optimization problem Eq.~(\ref{eqn:data_driven}) to explicitly include stability constraints:
\begin{equation}
\label{eqn:data_driven_stable}
\hat{A}_S = \arg\min_{\Theta} \sum_{j=1}^{N_A}  || Y^{(j)} - \Theta X^{(j)}||_F^2,\ \ s.t.\ \Theta \in \mathcal{S}_{A}
\end{equation}
While solving the optimization problem in Eq. (\ref{eqn:data_driven_stable}) can yield the constrained optimal solution which is guaranteed to be stable, it is known to be non-convex and requires approximations~\citep{boyd2004convex, mamakoukas2020memory}. 

\section{Post-Hoc Stability via Spectral Clipping}
\label{sec:clipping}

We investigate a simple post-hoc procedure, dubber ``Spectral Clipping (SC)", that surgically manipulates the spectrum of the system matrix $\hat{A}_{LS}$ in order to learn a stable LDS. Our motivation is to study computationally-cheap alternatives to standard approaches that explicitly solve a complex constrained optimization problem, such as the one in Eq.~(\ref{eqn:data_driven_stable}).


\begin{figure*}[t]
    \centering
\includegraphics[width=0.8\textwidth]{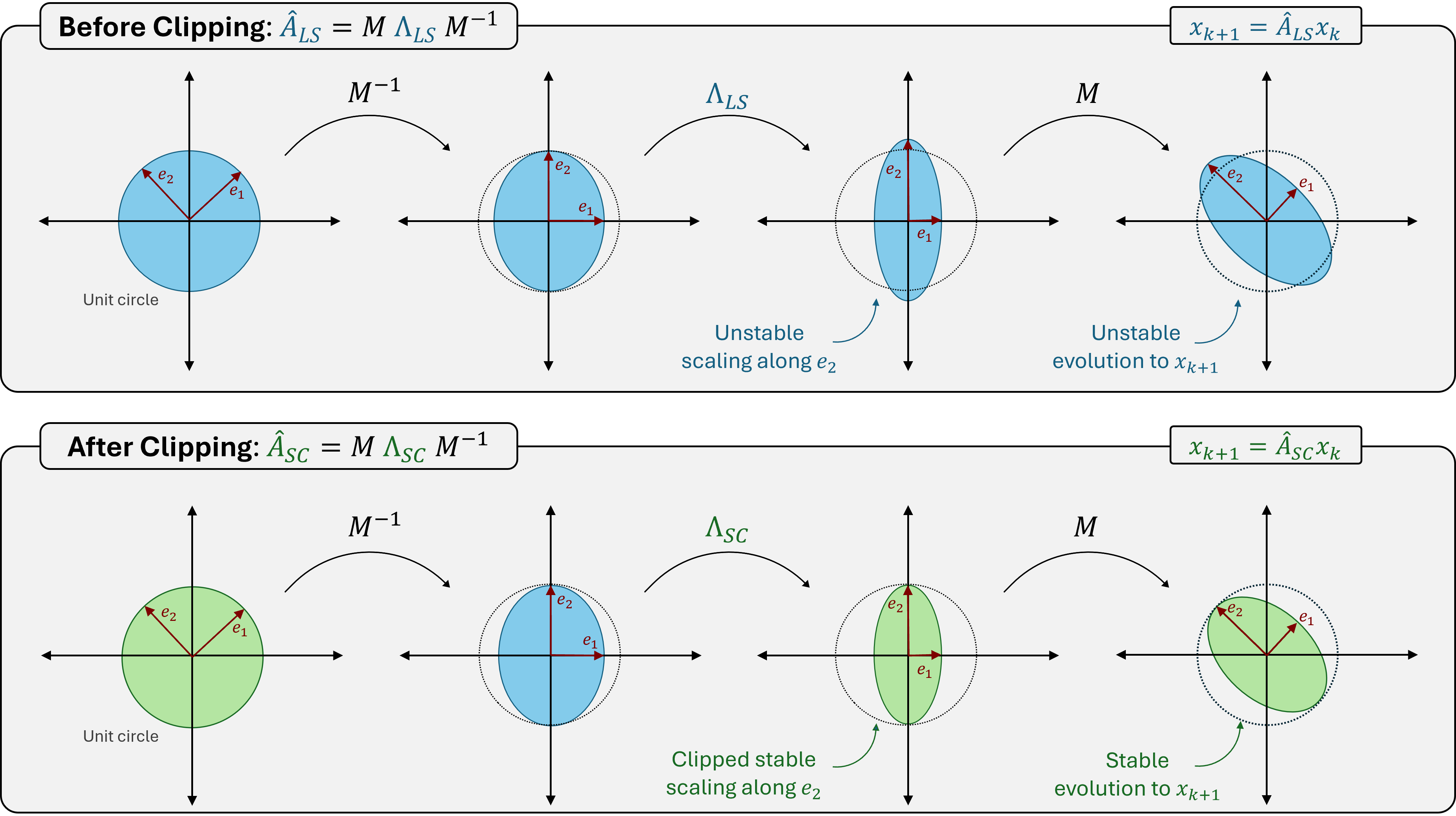}
        \caption{Geometric intuition: Spectral clipping enables stable evolution by clipping the spectral radius of the system matrix without altering its eigenvectors.}
        \label{fig:Geometric_interpretation}
\end{figure*}

\vspace{2pt}
\noindent\textbf{Algorithm}: Spectral Clipping (SC) involves the following steps:

\noindent\textit{Step 0. Unconstrained LS}: Estimate $\hat{A}_{LS}$ as the system matrix obtained from data by analytically solving the unconstrained least-squares problem (Eq.~(\ref{eqn:data_driven})). 

\noindent\textit{Step 1. Decomposition}: 
Perform eigen decomposition of the least squares solution: $\hat{A}_{LS} = M \Lambda M^{-1}$,  
where $\Lambda$ is a diagonal matrix with all the eigenvalues of $\hat{A}_{LS}$ on its diagonal, and $M$ is the modal matrix with the corresponding eigenvectors. 

\noindent\textit{Step 2. Spectral Clipping}: Identify all \textit{unstable} eigenvalues of $\hat{A}_{LS}$ in $\Lambda$ and clip their magnitudes: 
\newline $\bar{\Lambda}_{ii} = \left\{\begin{array}{ll}
\frac{\Lambda_{ii}}{\vert\Lambda_{ii}\vert}(1-\varepsilon), & \vert\Lambda_{ii}\vert \geq 1\\
\Lambda_{ii}, & \vert\Lambda_{ii}\vert < 1
\end{array}\right.$, where $\varepsilon \geq 0$ is an arbitrarily-small constant.

\noindent\textit{Step 3. Recomposition}: Reconstruct the system matrix based on the new $\bar{\Lambda}$ and the unaltered modal matrix $M$ to produce \textit{SC}'s estimate of the system matrix: $\hat{A}_{SC}\triangleq M \bar{\Lambda} M^{-1}$.

\noindent\textbf{Stability}: 
The reconstructed linear system after the above clipping procedure, $x_{t+1}=\hat{A}_{SC}\ x_t$ can be trivially shown to be asymptotically (marginally) stable When $\varepsilon >0$ ($\varepsilon=0$) since the maximum eigenvalue of $\hat{A}$ is strictly less than (equal to) one by construction~\citep{robinson2012introduction}.

\noindent\textbf{Dealing with non-diagonalizable Matrices}: A potential concern with the above approach is that $\hat{A}_{LS}$ may be non-diagonalizable. In practice, we can add tiny perturbations to the original $\hat{A}_{LS}$ to ensure a valid eigen-decomposition, which the built-in solvers in NumPy and MATLAB also perform. This is justified by the fact that a diagonalizable matrix can always be found arbitrarily close to any non-diagonalizable one (see Lemma.~\ref{lemma_dense} in Appendix.~\ref{sec:proof}). We also leverage this fact to show that the evolution of the system with a non-diagonalizable matrix and that of the system with the arbitrarily-close diagonalizable matrix can be made arbitrarily close (see proof in Appendix.~\ref{sec:proof}). As such, the \textit{SC} method can be applied to the arbitrarily-close diagonalizable matrix $\tilde{A}_{LS}$ to enforce stability of systems with non-diagonalizable matrices without significant approximation errors.



\noindent\textbf{Prediction accuracy}: Indeed, it is natural to suspect that this seemingly-naive approach would result in significantly poorer, albeit stable, approximations that produce worse predictions. However, as our experiments demonstrate, \textit{SC} surprisingly leads to similar accuracy compared to existing constrained optimization techniques (see Sec.~\ref{Experiments} for experiments across a wide range of domains).

\noindent\textbf{Expressivity-Stability tradeoff}: We note that the choice of $\varepsilon$ above represents an important tradeoff between the system's expressivity and stability. In particular, choosing $\varepsilon=0$ will yield \textit{marginally}-stable but highly-expressive systems. In contrast, choosing $\varepsilon>0$ will ensure \textit{strictly}-stable systems, with increasing values of $\varepsilon$ yielding faster exponential convergence and diminishing expressivity (see Appendix.~\ref{appendix:stability-expresivity tradeoff} for experimental validation).

\noindent\textbf{Intuition}: While \textit{SC}'s effectiveness may seem puzzling at surface, a deeper analysis offers some intuition. By altering only the eigenvalues while preserving eigenvectors, \textit{SC} essentially impacts the ``speeds'' or ``rates'' of the system's dynamics modes without qualitatively altering them. Specifically, \textit{SC} likely preserves the geometric structure of the state space as the invariant subspace spanned by the eigenvectors remains unchanged while modifying how trajectories behave within them. Since eigenvalues only determine whether trajectories grow, decay, or oscillate, altering their magnitudes shifts the stability margin without altering the system's directions of approach to divergence. Further, \textit{SC} intervenes only when necessary to ensure stability. By only altering eigenvalues that are larger than one, \textit{SC} reduces the rate of energy growth along specific paths (dictated by corresponding eigenvectors), which would have pushed the system towards instability without clipping. See Fig.~\ref{fig:Geometric_interpretation} for this geometric intuition. 

\noindent\textbf{Extensions}: We also extend \textit{SC} to settings beyond learning stable LDS. Specifically, we discuss both learning \textbf{nonlinear systems} (see Sec.~\ref{sec:Koopman_clipping}) and \textbf{systems with control inputs} (see Appendix.~\ref{sec:SC_control}).

\section{A Koopman-based Adaptation of SC to Nonlinear Systems}
\label{sec:Koopman_clipping}

We also investigate the potential utility of \textit{SC} when learning nonlinear dynamical systems (NLDS) by combining it with Koopman operator-based techniques~\citep{Koopman1931Koopman,mezic2020koopman}.

Let $\xi_{k+1}=f(\xi_k)$ be an unknown discrete-time NLDS with $\xi_k \in \Xi \subseteq \mathbb{R}^n, \forall k\in\mathbb{N}$ and $D_f = \{[\xi_1^{(j)}, \xi_2^{(j)}, \cdots, \xi_T^{(j)}]\}_{j=1}^{N_f}$ be a dataset containing $N_f$ trajectories of the NLDS. Koopman operator-based learning approaches learn to approximately encode the NLDS as $\mathrm{z}_{k+1}=\hat{K}\ \mathrm{z}_k$, where $\hat{K} \in \mathbb{R}^{m \times m}$ is the learned finite-dimensional approximation of the Koopman operator. $\mathrm{z}_k = \phi(\xi_k) \in \mathbb{R}^m$ is the ``lifted" latent state produced by the so-called \textit{``lifting" function} $\phi:\Xi \rightarrow \mathcal{Z}$ with  $m \gg n$, and $\mathcal{Z}$ is the approximated Koopman-invariant subspace in which the dynamics appear linear~\citep{Koopman1931Koopman,mezic2020koopman}. 

While it is possible to learn or design an encoder and a decoder to lift and reconstruct the original state, other formulations allow for trivial decoding and enable further analysis. In particular, we investigate \textit{state-inclusive latent space} designs~\citep{proctor2018generalizing,han2023KODex,johnson2025heterogeneous} which append the original state to the nonlinear lifted embedding (i.e., $\mathrm{z_t}=[\phi(\xi_t)^T,\xi_t^T]^T$ with $\mathrm{z}_{k+1}=\hat{K}\ \mathrm{z}_k$). As such, a trivial linear map can recover the original states from the latent states.


\noindent\textbf{Spectral clipping of the Koopman matrix}: Given a learned lifted latent LDS ($\mathrm{z}_{k+1}=\hat{K}\ \mathrm{z}_k$), we can perform spectral clipping of the learned Koopman matrix $\hat{K}$ following the same procedure described in Sec.~\ref{sec:clipping}. This process will yield a Schur-stable Koopman matrix $\hat{K}_{SC}$ with eigenvalues whose magnitudes are guaranteed to lie within the unit circle.

\noindent\textbf{Open-loop predictive stability}
Clipping ensures that the learned latent linear predictor $z_{t+1}=K_{\mathrm{SC}}z_t$ satisfies
$\rho(K_{\mathrm{SC}})\le 1-\epsilon$ for some $\epsilon\ge 0$.
If $\epsilon>0$, then $K_{\mathrm{SC}}^t z_0\to 0$ as $t\to\infty$ for all $z_0$, so the latent
predictions remain bounded and decay to the origin.
If $\epsilon=0$, then $\rho(K_{\mathrm{SC}})\le 1$ implies boundedness only under additional
conditions (e.g., no defective Jordan blocks on the unit circle), and in general one can only
guarantee that the latent predictor is not exponentially unstable.
Since we employ a state-inclusive lifting (i.e., the original state appears as a subvector of $z$),
the decoded open-loop predictions inherit the same boundedness/decay properties as the latent
trajectory.
We emphasize that this is a statement about the learned open-loop predictor, and does not imply
closed-loop stability of the underlying nonlinear dynamical system without further assumptions.

\section{Empirical Analyses}
\label{Experiments}

We evaluated \textit{SC} in terms of computation efficiency, stability, and expressivity when learning LDS from data and compare its performance against existing methods. Note that for all experiments, we use $\epsilon =0$ unless indicated otherwise (see Appendix.~\ref{appendix:stability-expresivity tradeoff} for the effects of different $\varepsilon$ values).
\subsection{Experimental Design}

\noindent\textbf{Compute}:
\label{sec:hardware}
We conducted all experiments on Alienware Aurora R13, with 10-core 12th Gen Intel Core i7-12700KF 1.0-GHz CPU with 32G RAM.

\noindent\textbf{Baselines}: We compared \textit{SC} against the following baselines: \textit{i)} \textit{LS}: vanilla least squares, \textit{ii)} (\textit{CG}): constraint generation~\citep{boots2007constraint}, \textit{iii)} \textit{WLS}: weighted least squares~\citep{huang2016learning}, and \textit{iv)} \textit{SOC}: a characterization approach~\citep{mamakoukas2020memory, mamakoukas2023learning}. 

\noindent\textbf{Datasets}: We evaluated all approaches on the following datasets.

\noindent\textit{i) UCSD}~\citep{chan2005probabilistic}: This dataset contains 254 highway traffic videos, each consisting of 48–52 frames of size $48\times48$. Following \citet{mamakoukas2020memory}. we apply SVD to each sequence and retain subspaces of dimension $r \in \{3, 30\}$ and reduce dimensionality. 

\noindent\textit{ii) DTDB}~\citep{hadji2018new}: A dataset of 285 dynamic texture movements, such as turbulence, wavy motions, etc. Each video consists of varying number of frames. As we did for the UCSD dataset, we leveraged SVD to lower the state dimension to $r = 300$. 

\noindent\textit{iii) DexManip}~\citep{han2023KODex, Rajeswaran2018DAPG}: A robotic dexterous manipulation dataset containing demonstrations of a 23-DoF multi-fingered Adroit hand performing four complex tasks in MUJOCO~\citep{Todorov2012MUJuCO}. 

\noindent\textit{iv) Franka Panda}~\citep{gaz2019dynamic}: A dataset of Franka Panda Robot Arm's movements as a consequence of different control inputs. Results are reported in Appendix~\ref{sec:control}.

\subsection{\textit{SC} significantly reduces computation \& memory burden}
\label{sec:time_memory}
\begin{wraptable}[6]{r}{0.5\textwidth} 
\vspace{-5pt} 
\centering
\resizebox{0.5\textwidth}{!}
{
\begin{tabular}{@{}lcccccc@{}}
\toprule
 & \textit{LS} & \textit{CG} & \textit{WLS} & \textit{SOC} & \textit{SC}
 \\ \toprule
Time complexity & $-$ & $O(n^6)$ & $O(n^6)$ & $O(n^3)$ & $O(n^3)$\\
Space complexity & $O(n^2)$ & $O(n^4)$ & $O(n^4)$ & $O(n^2)$ & $O(n^2)$\\
 \bottomrule
\end{tabular}
}
\captionsetup{skip=2pt}
\caption{Space and Time complexity of each method with $n$ being the system dimension.}
\label{tab_space_complexity}
\end{wraptable}
We begin by comparing the computational efficiency of each approach.
We report the average computation time across the 254 input sequences of the \textit{UCSD} dataset at each dimension in Fig.~\ref{fig:ucsd_dtdb} (a). The results show that 
\textit{SC} is significantly more time-efficient than \textit{WLS}, \textit{CG}, and \textit{SOC}. In Fig.~\ref{fig:ucsd_dtdb} (b), we report the results on the \textit{DTDB} dataset.
These results also indicate that the computation time of \textit{SC} is significantly and consistently shorter than that of \textit{SOC}, showing that \textit{SC} is more scalable w.r.t. the input dimensions. 
Similarly, the memory usage results reported in Fig.~\ref{fig:ucsd_dtdb} (c) and Fig.~\ref{fig:ucsd_dtdb} (d) suggest that \textit{SC} requires significantly less memory compared to \textit{CG}, \textit{WLS}, and \textit{SOC}. The improved efficiency is a result of \textit{SC} only requiring eigen decomposition and simple algebraic calculations,
while others rely on iterative optimization. Though \textit{LS} is the naturally the most time-efficient, it does not guarantee stability.
\begin{figure*}[tb]
    \centering
\includegraphics[width=\textwidth]{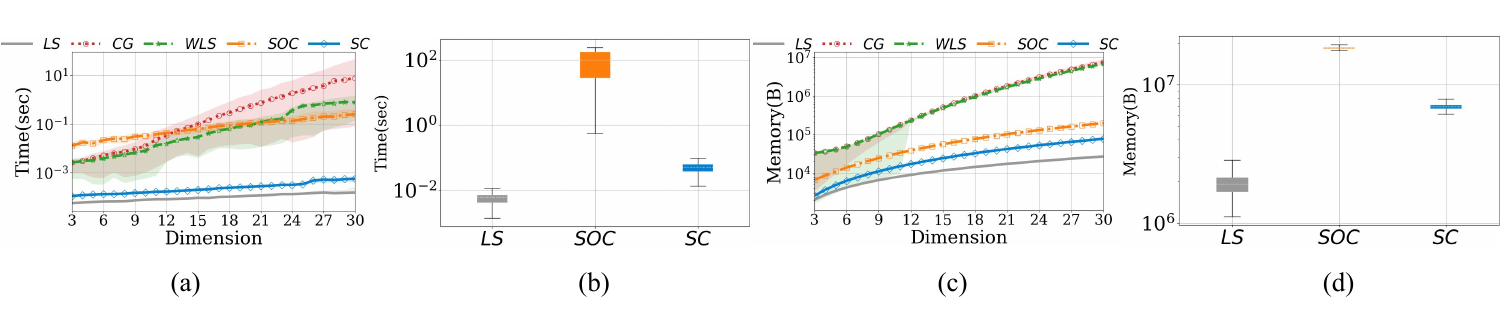}
        \captionsetup{skip=1pt}
        \caption{(a): Computation time on \textit{UCSD} dataset (solid lines indicate the median); (b): The boxplot of the average computation time on \textit{DTDB} dataset; (c): Memory usage on \textit{UCSD} dataset (solid lines indicate the median); (d): The boxplot of the average memory usage on \textit{DTDB} dataset.}
        \label{fig:ucsd_dtdb}
\end{figure*}

In addition to empirical results on the two datasets, we report the time and space complexity of each method in Table.~\ref{tab_space_complexity}. For time complexity, since all methods modify the original \textit{LS} solution to enforce stability, we compare only the cost incurred after solving \textit{LS}. For iterative methods (\textit{CG}, \textit{WLS}, and \textit{SOC}), the shown time complexity corresponds to a single iteration. In contrast, \textit{SC} is non-iterative, and its overall time complexity is shown. These results further suggest that \textit{SC}'s time and memory efficiency is comparable or better than those of the baselines. 

\subsection{\textit{SC} does not sacrifice predictive accuracy}

\begin{wrapfigure}[11]{r}{0.5\textwidth} 
\vspace{-8pt} 
\centering
    \begin{minipage}{0.5\textwidth}
    \centering
        \includegraphics[width=\textwidth]{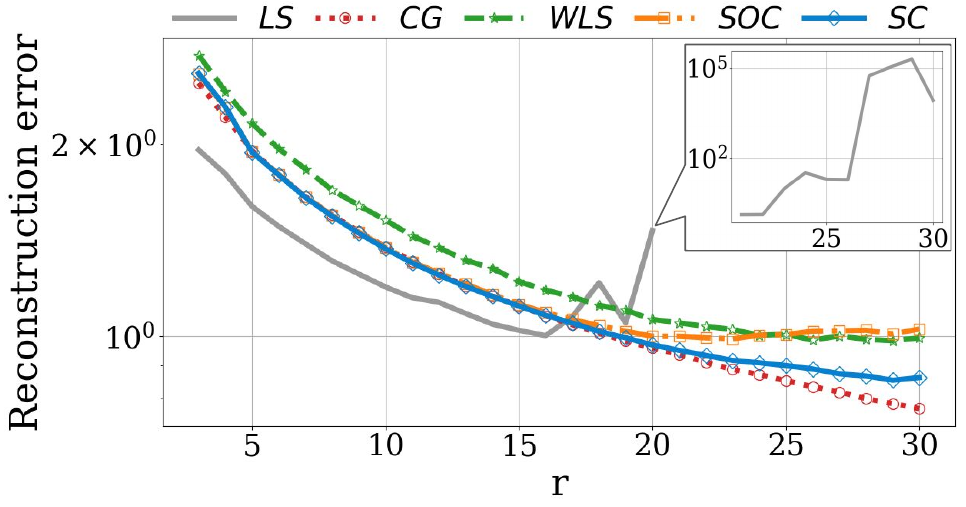}
        \captionsetup{skip=2pt}
        \caption{\small{Average reconstruction error evaluated on different subspace dimensions of the \textit{UCSD} dataset.}}
        \label{fig:error_ucsd_dim}
    \end{minipage}
\vspace{-20pt} 
\end{wrapfigure}
To evaluate if SC's improved efficiency comes at the expense of predictive accuracy, we compared the prediction accuracy of LDSs produced by each approach.
We first present the average reconstruction error across all states at each time step for the \textit{UCSD} dataset in Fig.~\ref{fig:ucsd_dtdb_error} (left). For illustrative purposes, we select $r=3, 20$ and $30$ as examples. We find that as the input dimension increases (e.g., $r=30$), \textit{LS} fails to reconstruct the input states, with the error increasing significantly over time. A similar trend can be observed in Fig.~\ref{fig:error_ucsd_dim}, where we show the average reconstruction error over all states and time steps of each method at each training dimension. In both cases, \textit{SC} consistently achieves results comparable to \textit{CG}, \textit{WLS}, and \textit{SOC}, while being orders of magnitude faster for computation and highly efficient in memory usage (as discussed in Section.~\ref{sec:time_memory}).


In Fig.~\ref{fig:ucsd_dtdb_error} (right), we report how prediction error varies across time for the \textit{DTDB} dataset as length of the input sequence given to the model before predictions. Note that due to the high dimension of the state space in this dataset, only \textit{LS}, \textit{SOC}, and \textit{SC} are applicable. Both \textit{CG} and \textit{WLS} involve solving a Kronecker product, which has a significantly large space complexity leading to memory overflow (i.e., $O(n^4)$ as shown in Table.~\ref{tab_space_complexity}, $n$ is the size of the state space).

The results indicate that \textit{SC} generates marginally more accurate predictions than \textit{SOC}, and its error grows slower as the prediction horizon increases. While \textit{LS} achieves the lowest error, it does not guarantee stability, a crucial property for reliable long-horizon predictions (see Section.~\ref{sec:long_term}). 

\begin{figure*}[!t]
    \begin{minipage}{0.5\textwidth}
    \centering
        \includegraphics[width=\textwidth]{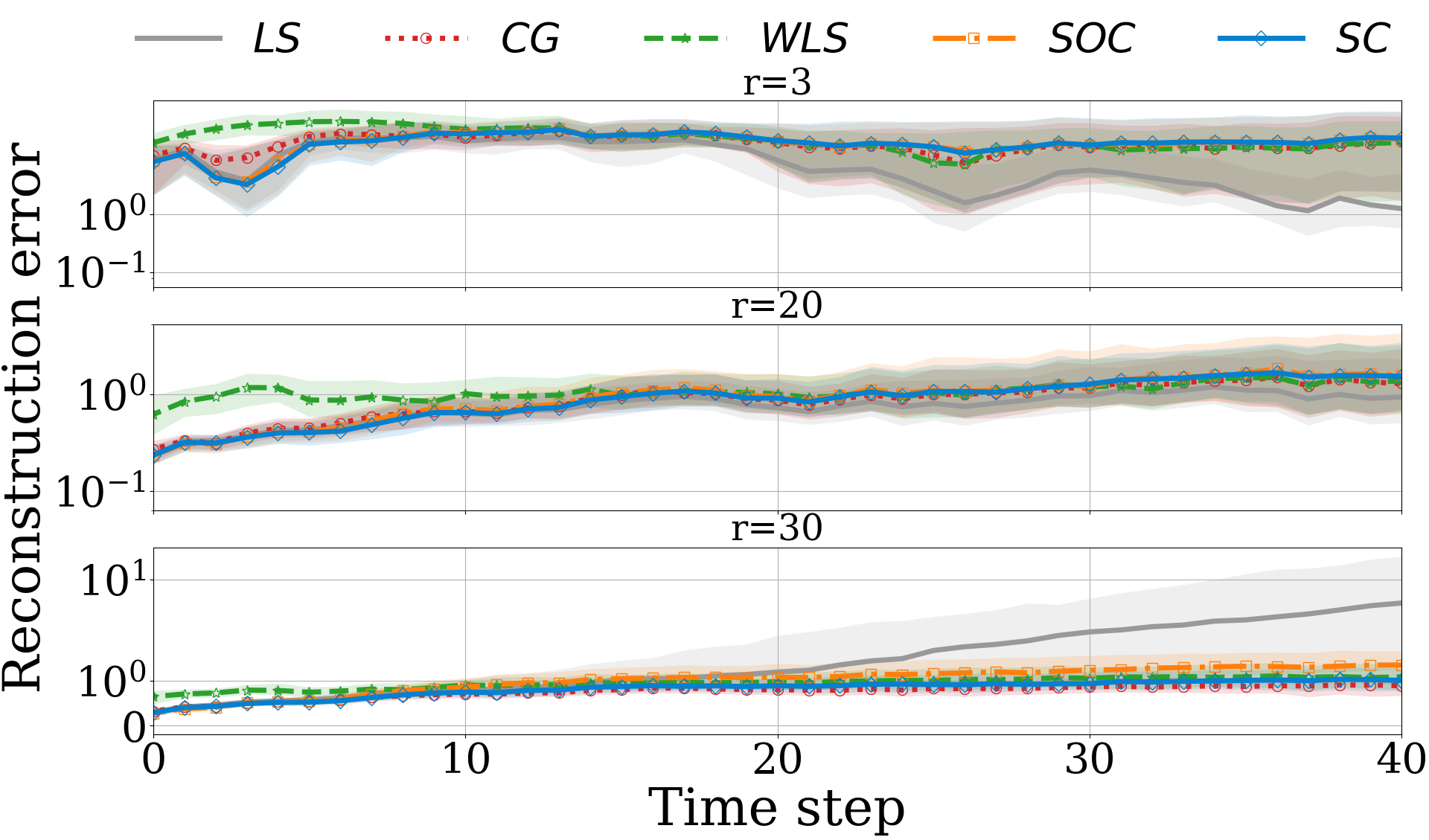}
    \end{minipage}
    \begin{minipage}{0.5\textwidth}
    \centering
        \includegraphics[width=\textwidth]{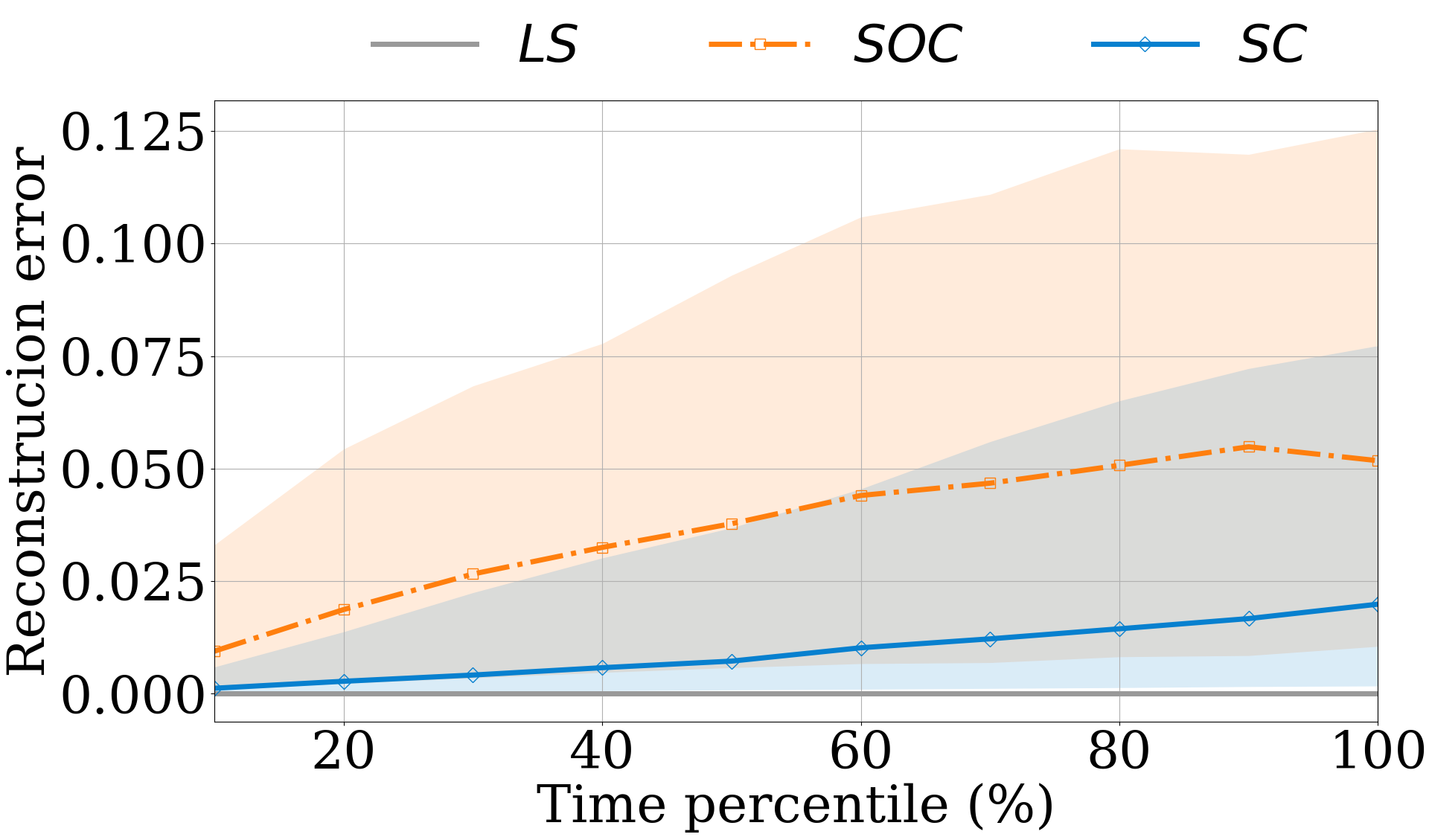}
    \end{minipage}
    \captionsetup{skip=2pt}
    \caption{Left figure: average reconstruction error evaluated on different subspace dimensions of the \textit{UCSD} dataset. Shaded areas indicate quartile values over all sequences. Right figure: average reconstruction error evaluated at each time step percentile of the \textit{DTDB} dataset. Shaded areas indicate quartile values over all sequences.}
    \vspace{-1.4em}
    \label{fig:ucsd_dtdb_error}
\end{figure*}

\vspace{-0.3cm}
\subsection{\textit{SC} enables consistent long-horizon predictions}
\label{sec:long_term}
We evaluated that \textit{SC} had the ability to make sensible long-horizon predictions that extend beyond the time-horizon of the training data.
To verify this, we extend the rollout horizon to both $H = 500$ and $H = 3000$ time steps for the 285 LDSs learned from the \textit{DTDB} dataset. Ideally, the learned system would be able to repeat the dynamical pattern learned from data, avoid both instability and convergence (black/static pixels) despite the extended rollout horizon. 

\begin{wraptable}[8]{r}{0.65\textwidth} 
\vspace{-5pt} 
\centering
\resizebox{0.65\textwidth}{!}{
\begin{tabular}{@{}lccc@{}}
\toprule
 & \textit{Avg. NIQE ($H=500$) ($\downarrow$)} & \textit{Avg. NIQE ($H=3000$) ($\downarrow$)} &  \textit{Moving Ratio ($\uparrow$)} \\ \toprule
\textit{\LS{LS}} & \LS{904.6265}  & \LS{6355.5817} & \LS{62.11\%}\\
\textit{SOC} & 832.8063  & 4243.4279 & 22.11\%\\
\textit{SC ($\varepsilon=0$)} & \textbf{18.8769} & \textbf{1075.0140} & \textbf{54.39\%}\\ 
\textit{SC($\varepsilon=10^{-5}$)} & \textbf{18.8769} & \textbf{1075.0139} & \textbf{54.04}\%\\ 
\textit{SC($\varepsilon=10^{-2}$)} & 706.0687 & 10000.0 & 0\% \\ 
\midrule
\addlinespace[2pt] 
\textit{Training Set (Avg.)} & \multicolumn{2}{c}{18.8768}  & 100.0\% \\
\bottomrule
\end{tabular}
}
\captionsetup{skip=2pt}
\caption{Long-horizon prediction metrics on the \textit{DTDB} dataset.}
\label{tab_magnitued}
\end{wraptable}
To quantify long-term predictive performance, we use the following metrics:
i) \textit{NIQE quality score~\citep{mittal2012making})}: A common image quality metric assessing the meaningfulness of LDS predictions. Note that some poor predictions (e.g., when most pixel values blow up or converge to 0) can produce unbounded NIQE scores, so we substitute them with 10000; ii) \textit{Moving Ratio~\footnote{\textit{Moving} is defined as $||F_{M}-F_{M-2}||_1>T$, where $\{F_{i}\}_{i=1}^{M}$ is a sequence of $M$ image frames and $T=9$.}}: The ratio of sequences whose final frames still exhibit the dynamic pattern, indicating whether LDS evolution preserves it.

We report the above metrics for the 285 sequences from the \textit{DTDB} dataset in Table.~\ref{tab_magnitued}. 
The results show that \textit{SC} best preserves long-term dynamic patterns while maintaining high prediction accuracy, particularly at 500-step predictions, which are nearly comparable to the training set images. This is likely since \textit{SC} can learn \textit{marginally-stable} systems ($\varepsilon = 0$) or stable systems with slow energy decay ($\varepsilon \rightarrow 0^+$), while \textit{LS} tends to produce unstable long-horizon predictions, and \textit{SOC} is bound to learn strictly stable systems that might be overly conservative and produce vanishing values. Refer to sample frame predictions in Fig.~\ref{fig:qual_dtdb_sc} of  Appendix~\ref{appendix:stability-expresivity tradeoff}. We also provide additional experiments and insights into \textit{SC}'s stability-expressivity tradeoff introduced by the choice of $\varepsilon$ in Appendix~\ref{appendix:stability-expresivity tradeoff}.

\subsection{\textit{SC}'s benefits extend to nonlinear systems when combined with the Koopman Operator}
\label{sec:kodex_performance}

We evaluated the effectiveness of combining the Koopman Operator theory~\citep{Koopman1931Koopman}, \textit{SC} can be effectively and efficiently applied to highly nonlinear robot manipulation tasks by evaluating SC on \textit{DexManip}~\citep{Rajeswaran2018DAPG}, a robotic dexterous manipulation dataset containing demonstrations of four complex tasks using a 23-DoF multi-fingered hand. Though the original state evolution is nonlinear, it was recently demonstrated that the evolution can be encoded using a linear system when ``lifted" to a higher-dimensional space (e.g., $r \approx 750$) that approximates the Koopman invariant subspace~\cite{han2023KODex}. Four tasks are briefly described as follows: i) \textit{Pen Reorientation}: The robot hand is tasked with orienting a pen to a random target rotation angles, ii) \textit{Door Opening}: The robot hand is tasked with opening a door from a random starting position, iii) \textit{Tool Use}: The robot hand is tasked with picking up the hammer and driving the nail into the board placed at a random height, and iv) \textit{Object Relocation}: The robot hand is tasked with moving a cylinder to a random target location (see Appendix.~\ref{sec:kodex} for the details of each dataset). Similar to the \textit{DTDB} dataset, only \textit{LS}, \textit{SOC}, and \textit{SC} are applicable due to the high dimensionality of the lifting space.
\begin{table}[h!]
    \centering
   \resizebox{0.9\linewidth}{!}
   {
    \begin{tabular}{l|ccccccccccccc}
        \toprule
        \multirow{2}{*}{} & 
        \multicolumn{3}{c}{\textbf{Computation Time} (sec, $\downarrow$)} & 
        \multicolumn{3}{c}{\textbf{Memory Usage} (MB, $\downarrow$)} & 
        \multicolumn{3}{c}{\textbf{Success Rate} (\%, $\uparrow$)} & 
        \multicolumn{3}{c}{\textbf{Safety Rate} (\%, $\uparrow$}) \\
         \textbf{Datasets} & \textit{\LS{LS}} & \textit{SOC} & \textit{SC}
        & \textit{\LS{LS}} & \textit{SOC} & \textit{SC}
        & \textit{\LS{LS}} & \textit{SOC} & \textit{SC}
        & \textit{\LS{LS}} & \textit{SOC} & \textit{SC} \\
        \midrule
        \textit{Reorient} & \LS{0.3} & 7308.9 & \textbf{0.4} & \LS{122.3} & 184.6 & \textbf{125.0} & \LS{69.5} & 52.7 & \textbf{75.6} & \LS{100.0} & \textbf{100.0} & \textbf{100.0} \\
        \textit{Door} & \LS{0.3} & 6.3 & \textbf{0.4} & \LS{96.2} & 151.1 & \textbf{98.6} & \LS{75.3} & 0.0 & \textbf{76.6} & \LS{100.0} & \textbf{100.0} & \textbf{100.0} \\
        \textit{Tool} & \LS{0.6} & 192.8 & \textbf{0.8} & \LS{224.6} & 331.7 & \textbf{257.2} & \LS{55.5} & 0.0 & \textbf{94.0} & \LS{0.0} & \textbf{100.0} & \textbf{100.0} \\
        \textit{Reloc} & \LS{0.7} & 6608.1 & \textbf{0.8} & \LS{215.3} & 321.3 & \textbf{247.5} & \LS{96.0} & 48.6 & \textbf{93.8} & \LS{100.0} & \textbf{100.0} & \textbf{100.0} \\
        \bottomrule
    \end{tabular}
    }
\captionsetup{skip=2pt}
\caption{\textit{SC} not only alleviates the computational and memory burden on the \textit{DexManip} datasets, but also enables successful and safe robot manipulation.}
\label{table:kodex_task_all_results}
\end{table}


We first report the computation time and memory usage in Table.~\ref{table:kodex_task_all_results}. As seen, \textit{SC} also alleviates the computational and memory burden compared to the \textit{SOC} method. Note that the values for LS method are marked in gray as this method does not guarantee stability. We then show the average reconstruction error on the \textit{DexManip} datasets in Fig.~\ref{fig:kodex_all_results}, where \textit{SC} demonstrates the best performance and \textit{LS} results in significantly larger errors on the \textit{Tool} dataset due to its lack of stability.

\begin{figure*}[!thb]
    \centering
    \includegraphics[width=0.7\textwidth]{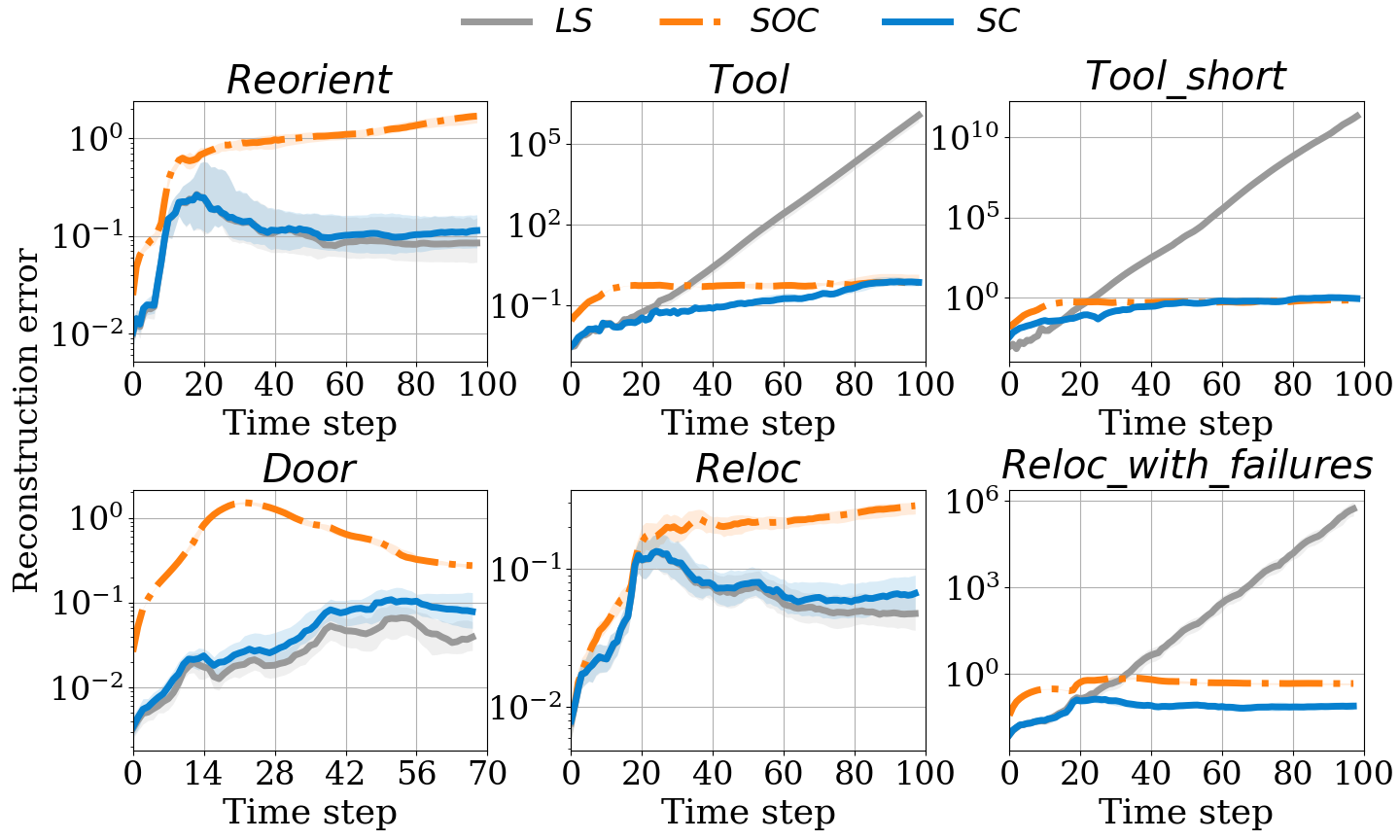}
    \captionsetup{skip=2pt}
    \caption{Average reconstruction error on the  \textit{Reorient}, \textit{Door}, \textit{Tool}, \textit{Reloc}, \textit{Tool\_short}, and \textit{Reloc\_with\_failures} datasets. Solid lines represent mean values, and shaded areas indicate quartile values over the rollout trajectories. }
    \vspace{-0.7em}
    \label{fig:kodex_all_results}
\end{figure*}

To further underscore the importance of ensuring stability when learning LDSs for robotic dexterous manipulation tasks in the \textit{DexManip} datasets, we also compute the task success rates and the safety rates (both reported in Table.~\ref{table:kodex_task_all_results}) for each task when tracking the generated robot hand trajectories. The definition of these two metrics are included in Appendix.~\ref{sec:kodex}.


For both tasks, \textit{SC} learns stable LDSs that achieve the highest task success rates and complete the tasks safely. In contrast, \textit{SOC} underperforms due to the fact it learns overly conservative LDSs that lack the ability to generate diverse trajectories. Note that while \textit{LS} seems to achieve the desired performance on the \textit{Reloc} task, its effectiveness drops significantly when trained on a corrupted dataset as shown in Section.~\ref{sec:corrupted}.
\vspace{-0.3cm}
\subsection{\textit{SC} offers robustness to unsuccessful and shorter demonstrations}
\label{sec:corrupted}

In real-world settings, the collected demonstrations may be imperfect, unlike those in the original \textit{DexManip} datasets. Hence, we introduce two additional datasets: (i) \textit{Tool\_short}, where each demonstration is shorter, posing challenges in learning stable LDSs with a limited horizon, and (ii) \textit{Reloc\_with\_failures}, which includes a few failed demonstrations (see Appendix.~\ref{sec:kodex} for more details).
\begin{table}[h!]
    \centering
   \resizebox{\linewidth}{!}{
    \begin{tabular}{l|ccccccccccccc}
        \toprule
        \multirow{2}{*}{} & 
        \multicolumn{3}{c}{\textbf{Computation Time} (sec, $\downarrow$)} & 
        \multicolumn{3}{c}{\textbf{Memory Usage} (MB, $\downarrow$)} & 
        \multicolumn{3}{c}{\textbf{Success Rate} (\%, $\uparrow$)} & 
        \multicolumn{3}{c}{\textbf{Safety Rate} (\%, $\uparrow$}) \\
        \textbf{Datasets} & \textit{\LS{LS}} & \textit{SOC} & \textit{SC}
        & \textit{\LS{LS}} & \textit{SOC} & \textit{SC}
        & \textit{\LS{LS}} & \textit{SOC} & \textit{SC}
        & \textit{\LS{LS}} & \textit{SOC} & \textit{SC} \\
        \midrule
        \textit{Tool\_short} & \LS{0.3} & 3470.2 & 0.4 & \LS{80.2} & 280.5 & 112.8 & \LS{3.3} & 0.0 & 90.7 & \LS{0.0} & 100.0 & 100.0 \\
         \textit{Reloc\_with\_failures} & \LS{0.6} & 4564.6 & 0.9 & \LS{216.1} & 322.1 & 248.4 & \LS{9.6} & 0.0 & 94.4 & \LS{0.0} & 100.0 & 100.0\\
        \bottomrule
    \end{tabular}
    }
    \captionsetup{skip=2pt}
    \caption{\textit{SC} can efficiently learn stable manipulation skills from shorter and failed demonstrations.}
    \label{table:kodex_task_hard}
\end{table}


From Table.~\ref{table:kodex_task_hard}, and it is evident that \textit{SC} is more efficient than \textit{SOC} in terms of compute and memory. In \textit{Tool\_short} and \textit{Reloc\_with\_failures} of Fig.~\ref{fig:kodex_all_results}, we show the average reconstruction error over the original uncorrupted datasets. From these results, we can observe that 
the LDSs learned by the \textit{LS} approach tend to be highly unstable, as evidenced by the significantly increasing error over time. A similar conclusion can be drawn from the task success rates and safety rates (both reported in Table.~\ref{table:kodex_task_hard}), where \textit{SC} demonstrates significantly superior performance.
This indicates the importance of enforcing stability in learned LDSs to effectively handle corrupted demonstrations. 

We also reveal the connections between \textit{SC} and the Koopman Operator's spectral properties. Following the algorithm described in~\cite{rowley2009spectral}, we can rollout the systems with only subsets of the Koopman modes 
and visualize the hand motions to investigate the effectiveness of \textit{SC} upon each subset. See detailed experiments in Appendix.~\ref{sec:koopman_modes}, where we demonstrated how \textit{SC} effectively stabilizes unstable modes and keeps stable modes. 

\section{Conclusions}
\label{sec:conclusion}
Spectral clipping (SC) is a simple yet surprisingly-effective post-hoc method for ensuring the stability of learned linear time-invariant dynamical systems through extensive experiments across multiple domains. Further, SC can be combined with Koopman-based methods to learn nonlinear dynamical systems and ensure open-loop predictive stability. Key benefits of SC include i) significant gains in computation efficiency without sacrificing prediction accuracy, ii) stable long-term predictions, and iii) robustness to datasets with failed or shorter demonstrations. 
A key limitation of this work is that we limit our investigation to deterministic systems. Extending these experiments to stochastic systems and automating the choice of the stability margin $\varepsilon$
are promising avenues for future work. 
\newpage
\bibliography{references}


\clearpage
\appendix

\begin{center}
    {\LARGE{Appendices}}
\end{center}
\section{Approximation error bounds for non-diagonalizable system matrices}\label{sec:proof}

\begin{lemma}
\label{lemma_dense}
For any $\gamma > 0$ and any non-diagonalizable matrix $\hat{A}_{LS}$, there exists a diagonalizable matrix $\tilde{A}_{LS} \in \mathbb{R}^{n \times n}$ s.t. $||\hat{A}_{LS} - \tilde{A}_{LS}||_1 < \gamma$.
\end{lemma}
\begin{proof}
Lemma.~\ref{lemma_dense} can be proved by the fact that the set of real matrices that are diagonalizable over $\mathbb{C}$ is dense in $\mathbb{R}^{n \times n}$. The details can be found from the matrix perturbation theory~\citep{horn2012matrix, trefethen2022numerical}.
\end{proof}

\begin{theorem}[Finite-horizon approximation of trajectories]
Let $K\in\mathbb{N}$ represent a fixed finite horizon of interest.
For any $\varepsilon>0$, any $A\in\mathbb{R}^{n\times n}$ (possibly non-diagonalizable), and any
$z\in\mathbb{R}^n$, there exists a matrix $\tilde A\in\mathbb{R}^{n\times n}$ that is diagonalizable
over $\mathbb{C}$ such that for all $k\in\{0,1,\dots,K\}$,
\[
\|A^k z-\tilde A^k z\|_2 \le \varepsilon.
\]
\end{theorem}

\begin{proof}
Let $E:=\tilde A-A$.  For any integer $k\ge 1$, the following identity holds without any
commutativity assumptions:
\[
\tilde A^k - A^k
= \sum_{i=0}^{k-1} \tilde A^{\,k-1-i}\,E\,A^i .
\]
Therefore, by submultiplicativity of the operator norm,
\[
\|\tilde A^k - A^k\|_2
\le \|E\|_2 \sum_{i=0}^{k-1}\|\tilde A\|_2^{k-1-i}\|A\|_2^{i}.
\]
Choose $\|E\|_2\le 1$, which implies $\|\tilde A\|_2\le \|A\|_2+\|E\|_2\le \|A\|_2+1=: \alpha$.
Then for all $k\le K$,
\[
\|\tilde A^k - A^k\|_2
\le \|E\|_2 \sum_{i=0}^{k-1}\alpha^{k-1-i}\alpha^i
= \|E\|_2\, k\, \alpha^{k-1}
\le \|E\|_2\, K\, \alpha^{K-1}.
\]
Hence,
\[
\|A^k z-\tilde A^k z\|_2 \le \|\tilde A^k-A^k\|_2\|z\|_2
\le \|E\|_2\, K\,\alpha^{K-1}\,\|z\|_2.
\]
Now choose
\[
\|E\|_2 \le \min\Bigl\{1,\ \frac{\varepsilon}{K\,\alpha^{K-1}\,\|z\|_2}\Bigr\}.
\]
Given the density of diagonalizable matrices in $\mathbb{R}^{n \times n}$ (see Lemma~1), we can achieve the bound in $E$ with an appropriate diagonalizable $\tilde A$.
This yields the desired bound for all $k\in\{0,\dots,K\}$.
\end{proof}

\noindent\textbf{Remark.}
The above theorem is specifically stated for a fixed finite horizon $K$.
Without additional assumptions (e.g., power-boundedness of $\hat{A}_{LS}$ and $\tilde{A}_{LS}$), a uniform bound
as $k\to\infty$ generally cannot be guaranteed since $\|A^k\|$ may grow with $k$.

\section{Spectrum Clipping for Controlled Systems}\label{sec:SC_control}
Consider a discrete-time controlled linear system as
\begin{equation}
\label{eqn:system_control}
\mathrm{x}_{t+1} = A \mathrm{x}_{t} + B\mathrm{u}_{t}, 
\end{equation}
where $\mathrm{x}_t \in \mathcal{X} \subset \mathbb{R}^n, \mathrm{u}_t \in \mathcal{U} \subset \mathbb{R}^m$ are the system state and control signal at time $t$, respectively. Similarly, $A \in \mathcal{S}_{A}$ is the unknown system matrix, and $\mathcal{S}_{A} \subset \mathbb{R}^{n \times n}$ is the subspace that contains stable matrices: $\mathcal{S}_{A} \triangleq \{A \in \mathbb{R}^{n \times n}| \max \{|\lambda_i(A)|\}^n_{i=1} <= 1\}$, and $B \in \mathcal{S}_{B}$ is the unknown control matrix, and $\mathcal{S}_{B} \subset \mathbb{R}^{n \times m}$. 


\textbf{Unconstrained Learning}: Now, suppose we have a dataset $D = [\mathrm{x}_1, \mathrm{u}_1, \mathrm{x}_2, \mathrm{u}_2, \cdots, \mathrm{x}_T]$. Let $Y = [\mathrm{x}_2, \mathrm{x}_3, \cdots, \mathrm{x}_T] \in \mathbb{R}^{n \times T-1}, X = [\mathrm{x}_1, \mathrm{x}_2, \cdots, \mathrm{x}_{T-1}] \in \mathbb{R}^{n \times T-1}$, and $U = [\mathrm{u}_1, \mathrm{u}_2, \cdots, \mathrm{u}_{T-1}] \in \mathbb{R}^{m \times T-1}$, then to learn the linear system from the dataset $D$ is to minimize the similar regression objective:

\begin{equation}
\label{eqn:data_driven_controlled}
\hat{A}_{LS}, \hat{B} = \arg\min_{\{\Theta, \Gamma\}} || Y - AX - BU||_F^2,
\end{equation}
where $\Theta \in \mathbb{R}^{n \times n}$ and $\Gamma \in \mathcal{S}_{B}$ are both the learnable parameters, and $||\cdot||_F$ is the matrix Frobenius norm.
Similarly, optimizing Eq.~(\ref{eqn:data_driven_controlled}) does not guarantee stability. 

\textbf{Learning Stable LDS with Control Inputs}: To this end, we can rewrite the optimization problem Eq.~(\ref{eqn:data_driven_controlled}) to include stability constraints:
\begin{equation}
\label{eqn:data_driven_stable_controlled}
\hat{A}_S, \hat{B} = \arg\min_{\{\Theta\in\mathcal{S}_{A}, \Gamma\}} || Y - AX - BU||_F^2,
\end{equation}
To find the approximate solution to Eq.~(\ref{eqn:data_driven_stable_controlled}), we first compute $\hat{A}_{LS}$ and $\hat{B}$ by optimizing the least square regression (Eq.~(\ref{eqn:data_driven_controlled})). Then we apply \textit{SC} method to $\hat{A}_{LS}$ and obtain $\hat{A}_{SC}$ as described in Section.~\ref{sec:clipping}, while keeping $\hat{B}$ unchanged.

The experiment results can be found in Appendix.~\ref{sec:control}.


\section{Tradeoff between stability and expressivity}
\label{appendix:stability-expresivity tradeoff}
As mentioned in Section.~\ref{sec:clipping}, the choice of $\varepsilon$ values establishes a trade-off between expressivity and stability. 
In the experiments reported below, we systematically analyze this fundamental trade-off.
We omit the analysis of compute and memory efficiency since $\varepsilon$ does not impact them.

\textbf{\textit{UCSD} and \textit{DTDB} dataset}: We first compute the \textit{Reconstruction error} of matrices clipped to different $\varepsilon$ values,
\begin{figure*}[!thb]
    \begin{minipage}[b]{0.5\textwidth}
    \centering
\includegraphics[width=\textwidth]{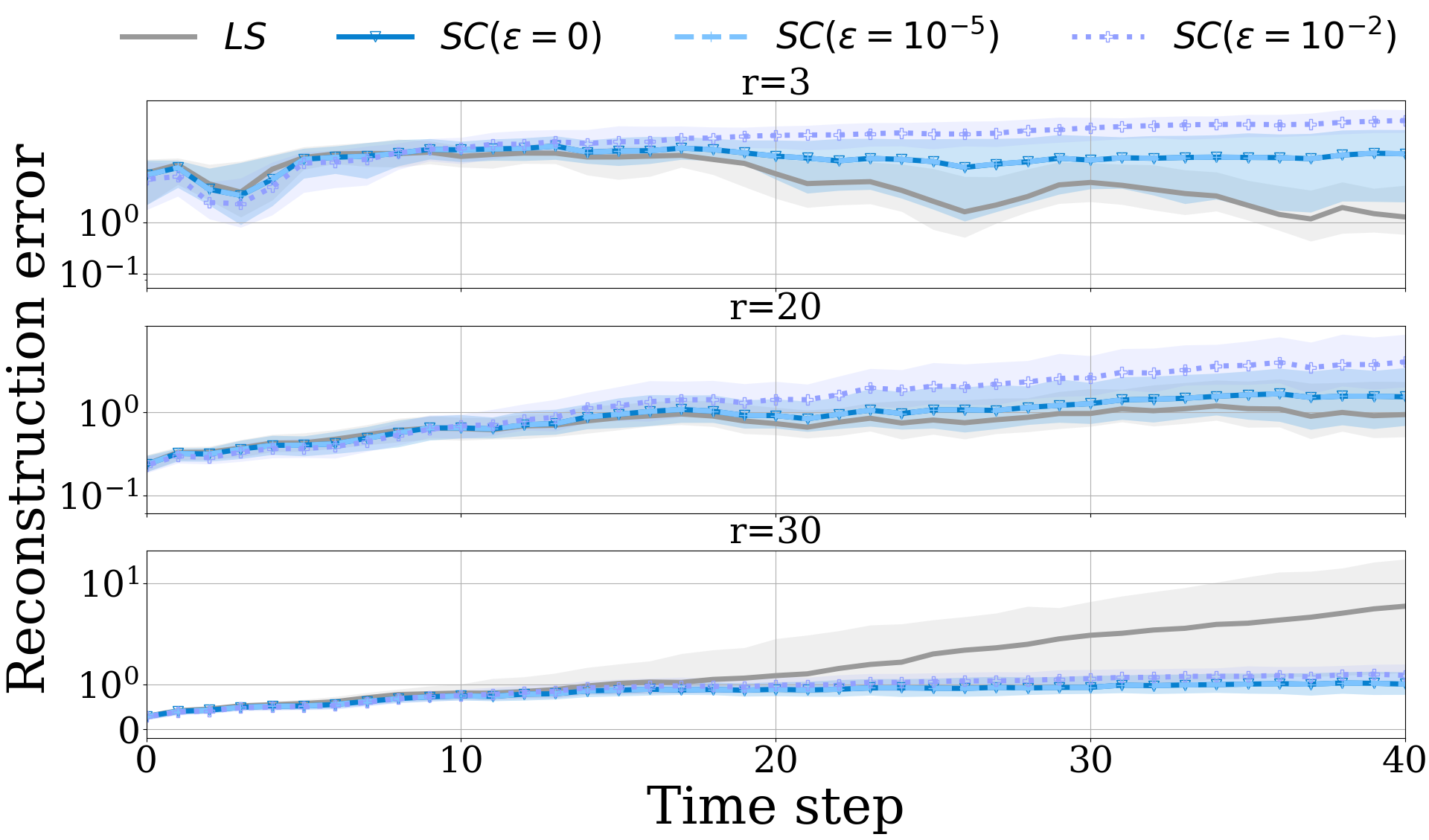}
    \end{minipage}
    \begin{minipage}[b]{0.5\textwidth}
    \centering
        \includegraphics[width=\textwidth]{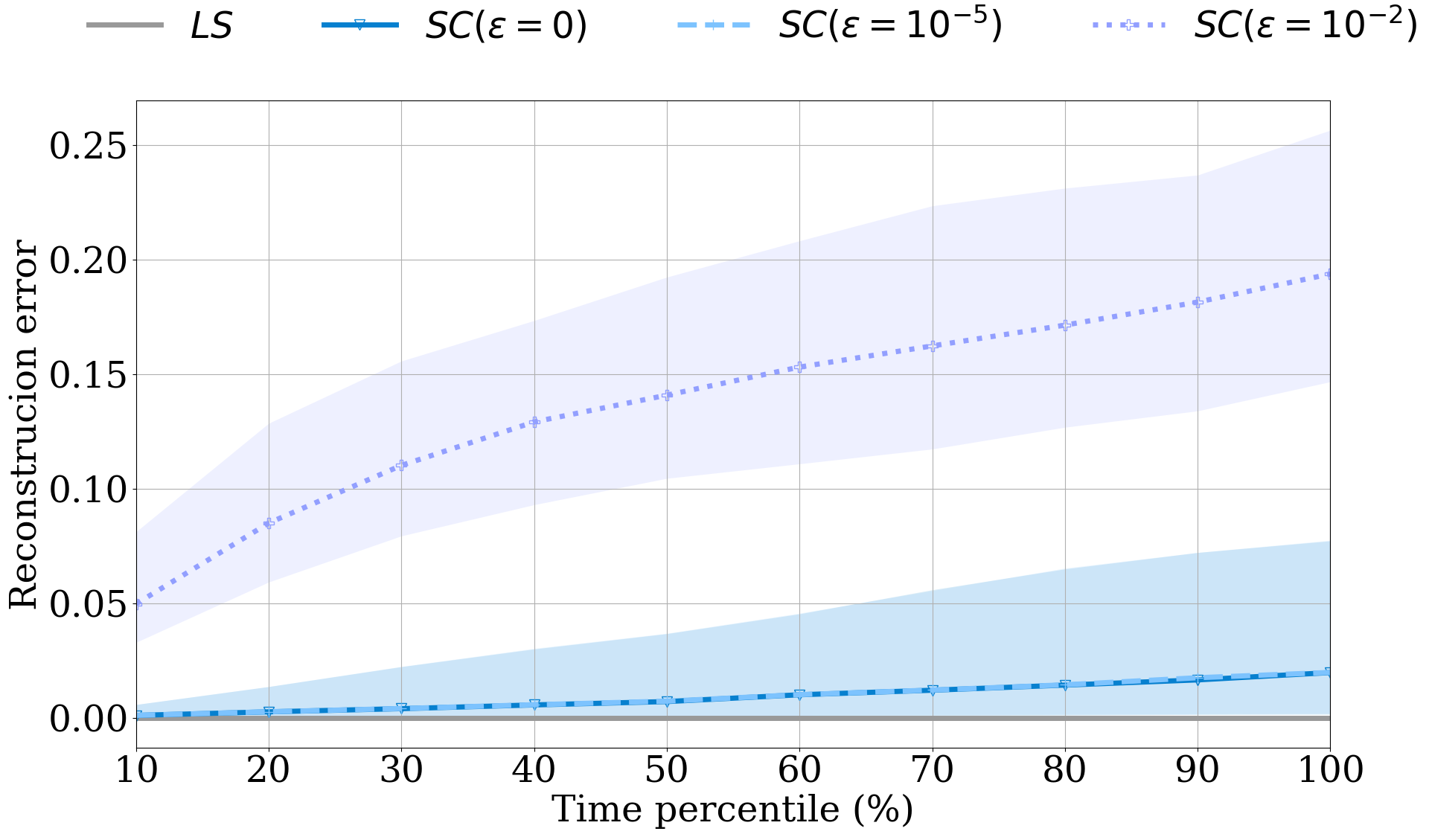}
    \end{minipage}
    \caption{Left figure: average reconstruction error evaluated on different subspace dimensions of the \textit{UCSD} dataset. Solid lines represent median values, and shaded areas indicate quartile values. Right figure: Average reconstruction error evaluated at each time step percentile of the \textit{DTDB} dataset. Solid lines represent median values, and shaded areas indicate quartile values.}
    \label{fig:ucsd_dtdb_error_eps}
\end{figure*}
\begin{figure*}[!thb]
    \centering
\includegraphics[width=0.5\textwidth]{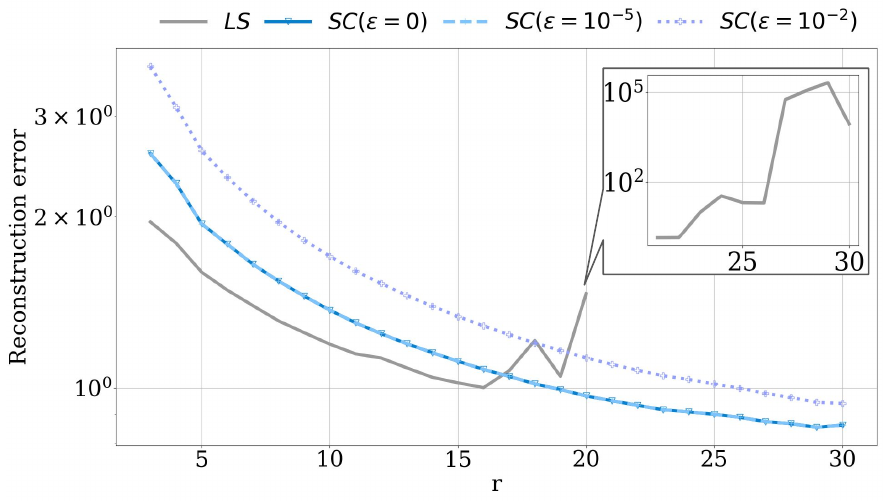}
        \caption{Average reconstruction error as a function of subspace dimensions of the \textit{UCSD} dataset.}
        \label{fig:error_ucsd_dim_epsilon}
\end{figure*}
and show the results in Fig.~\ref{fig:ucsd_dtdb_error_eps}. We can observe that the \textit{Reconstruction error} of \textit{SC}($\varepsilon=0$) and $SC(\varepsilon=10^{-5})$ are not distinguishable. 
However, if the $\varepsilon$ is chosen larger, i.e., $\varepsilon=10^{-2}$, the error becomes significantly larger, primarily because the learned system is too conservative and lacks expressivity.

\begin{figure*}[!h]
    \begin{minipage}[b]{\textwidth}
    \centering
        \includegraphics[width=\textwidth]{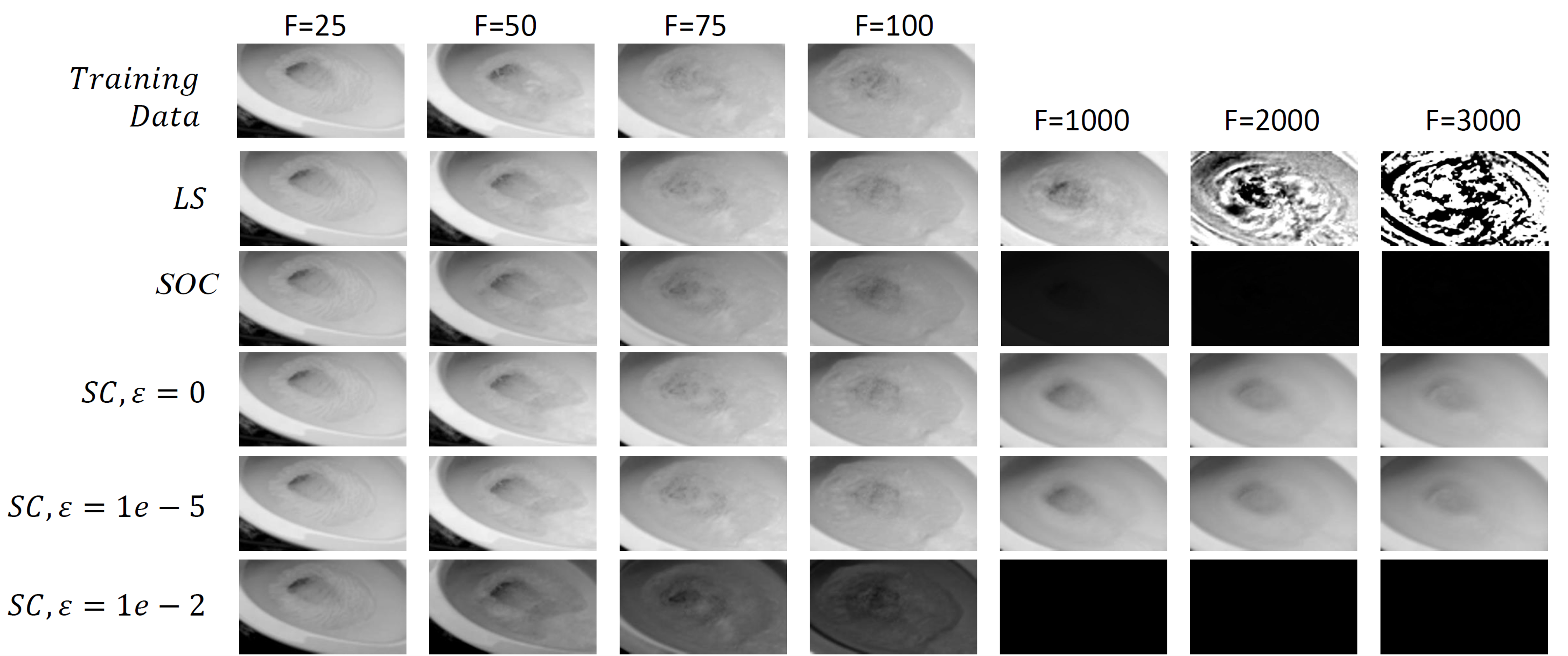}
        \captionsetup{labelformat=empty}
        \caption*{\small{\textit{DTDB Sequence 185: A flushing toilet.}}}
    \end{minipage}
    \begin{minipage}[b]{\textwidth}
    \centering
        \includegraphics[width=\textwidth]{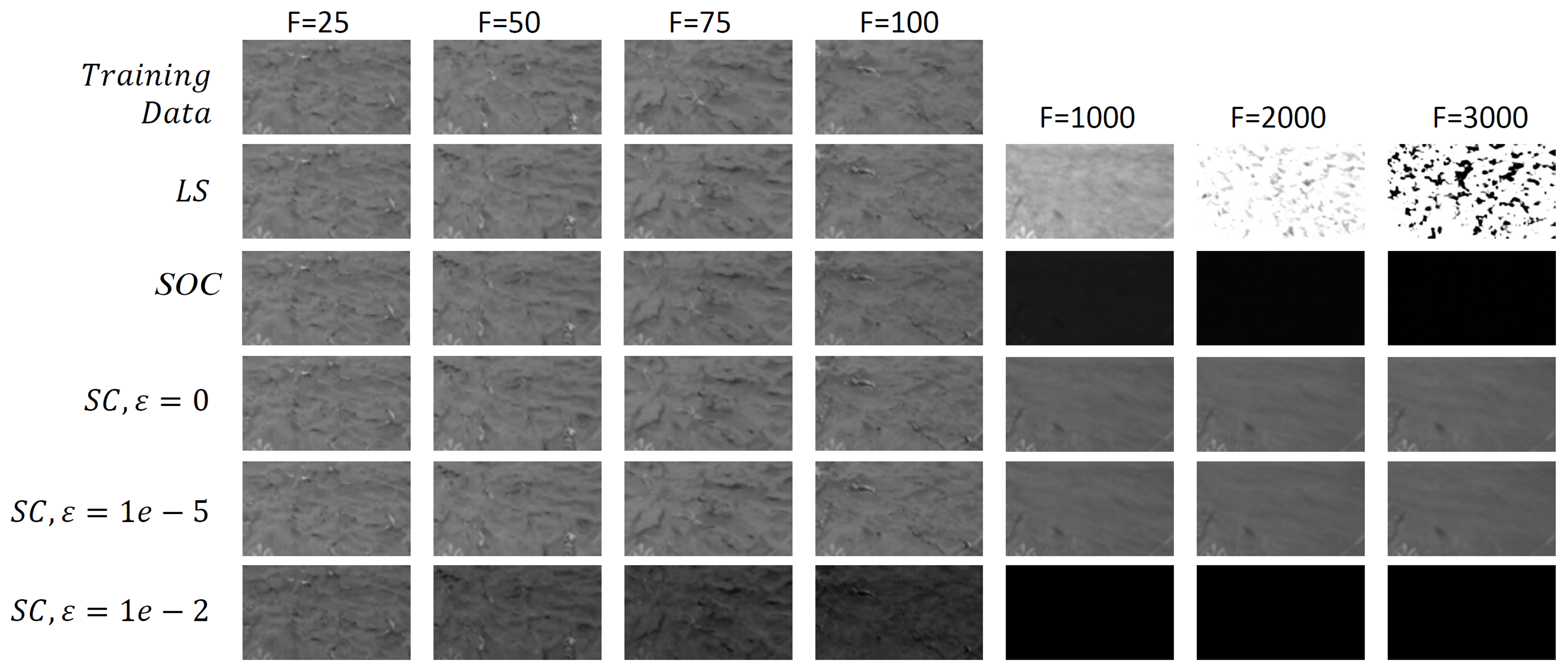}
        \captionsetup{labelformat=empty}
        \caption*{\small{\textit{DTDB Sequence 205: A running stream.}}}
    \end{minipage}
    \caption{Qualitative prediction results on two example sequences from DTDB dataset. 
    }
    \label{fig:qual_dtdb_sc}
\end{figure*}
Next, similar trend can be seen from the qualitative results of two selected sequences (\textit{DTDB} dataset) shown in Fig.~\ref{fig:qual_dtdb_sc}, where \textit{SC}($\varepsilon=0$) and \textit{SC}($\varepsilon=10^{-5}$) yield similar outcomes, while \textit{SC}($\varepsilon=10^{-2}$) lacks expressivity, resulting in turning black. In addition, the quantitative results reported in Table.~\ref{tab_magnitued_different_epsilon} also support this finding.
\begin{table}[!htb]
\centering
 \resizebox{0.75\linewidth}{!}{
\begin{tabular}{@{}lccc@{}}
\toprule
 & \textit{Average NIQE ($F=500$) ($\downarrow$)} & \textit{Average NIQE ($F=3000$) ($\downarrow$)} &  \textit{Moving ($\uparrow$)} \\ \toprule
\textit{\LS{LS}} & \LS{904.6265}  & \LS{6355.5817} & \LS{62.11\%}\\
\textit{SC($\varepsilon=0$)} & \textbf{18.8769}  & 1075.0140 & \textbf{54.39\%}\\
\textit{SC($\varepsilon=10^{-5}$)} & \textbf{18.8769} & \textbf{1075.0139} & 54.04\%\\ 
\textit{SC($\varepsilon=10^{-2}$)} & 706.0687 & 10000.0 & 0\%\\ 
\midrule
\addlinespace[2pt] 
\textit{Training Set} & \multicolumn{2}{c}{18.8768 (average over all training sequences)}  & 100.0\% \\
\bottomrule
\end{tabular}
}
\vspace{0.3em}
\caption{Quantitative metrics evaluating the predictions with varying $\varepsilon$ values on the \textit{DTDB} dataset.}
\label{tab_magnitued_different_epsilon}
\end{table}

\textbf{\textit{DexManip} dataset}: We find that the performance of \textit{SC}($\varepsilon=0$) and $SC(\varepsilon=10^{-5})$ are also nearly distinguishable, as shown in Fig.~\ref{fig:error_kodex_epsilon_ori} (\textit{Reconstruction error}) and Table.~\ref{table:kodex_sr_epsilon} and Table.~\ref{table:kodex_safe_epsilon} (\textit{Success \& Safety rates}). Also, we can observe that $SC(\varepsilon=10^{-2})$
lacks the expressivity needed to generate diverse motions, leading to a significant low task success rate, although it maintains a 100\% safety rate as all states converge to zero.

\begin{figure*}[!htb]
    \centering
    \includegraphics[width=0.7\textwidth,clip]{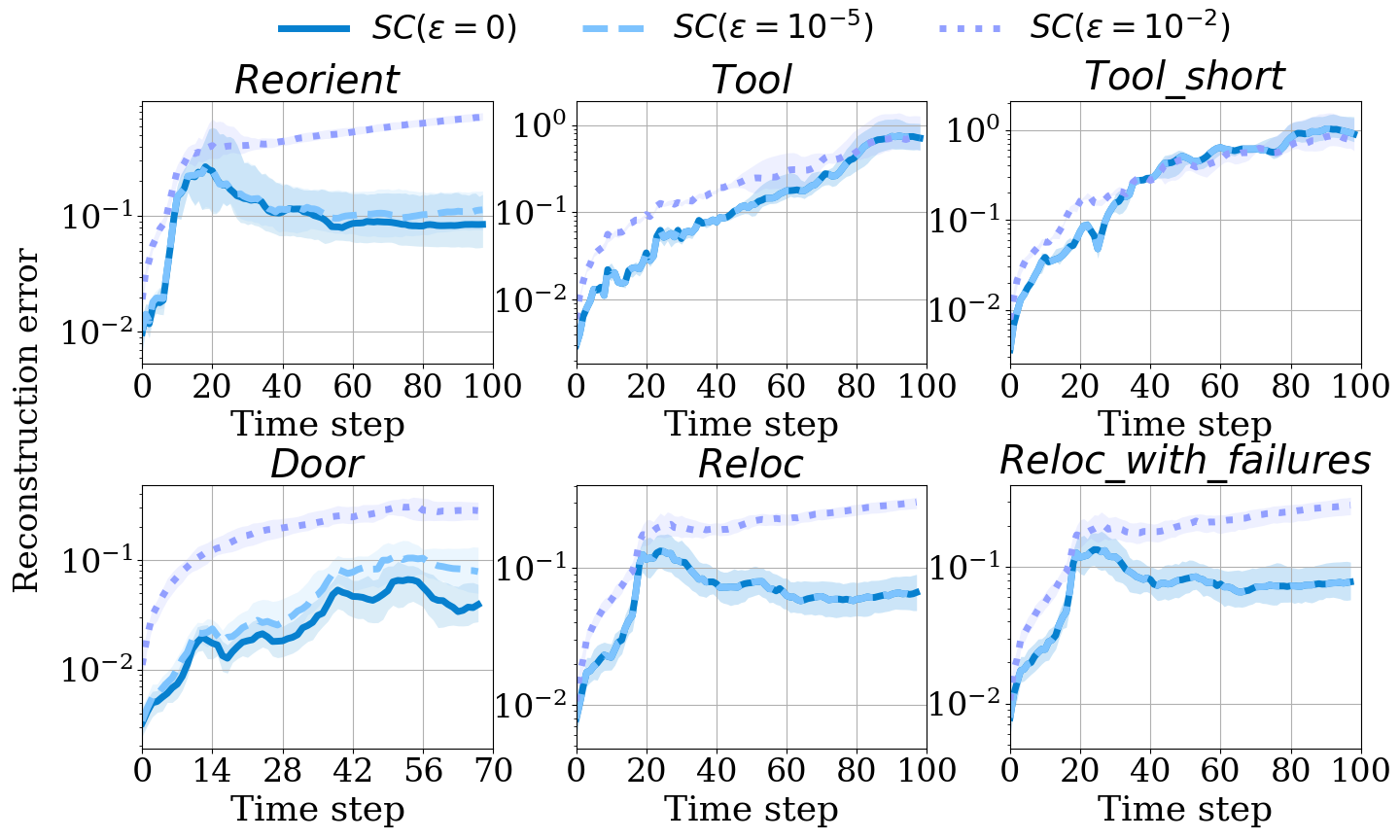}
    \caption{Average reconstruction error on the  \textit{Reorient}, \textit{Reloc}, \textit{Door}, \textit{Tool}, \textit{Reloc\_with\_failures} and \textit{Tool\_short} datasets. Solid lines represent mean values, and shaded areas indicate quartile values over the rollout trajectories.}
    \label{fig:error_kodex_epsilon_ori}
\end{figure*}

\begin{table}[!hbt]
\centering
\begin{tabular}{@{}lcccccc@{}}
\toprule
 & \textit{Reorient} & \textit{Door} & \textit{Tool} & \textit{Tool\_short} & \textit{Reloc} & \textit{Reloc\_with\_failures}\\ \toprule
\textit{\LS{LS}} & \LS{69.5\%} & \LS{75.3\%} & \LS{55.5\%} & \LS{3.3\%} & \LS{96.0\%} & \LS{9.6\%}\\
\textit{SC($\varepsilon=0$)} & 75.6\% & \textbf{76.6\%} & 
\textbf{94.0\%} & \textbf{90.7\%} & 93.8\% & \textbf{94.4\%}\\
\textit{SC($\varepsilon=10^{-5}$)} & 76.3\%& \textbf{76.6\%} & \textbf{94.0\%} & 89.0\% & \textbf{94.4\%} & 93.8\%\\
\textit{SC($\varepsilon=10^{-2}$)} & \textbf{83.2\%} & 12.0\% & 24.7\% & 0\% & 66.7\% & 42.9\%\\ \bottomrule
\end{tabular}
\vspace{0.3em}
\caption{Task Success Rate for learned policy with varying $\epsilon$ values.}
\label{table:kodex_sr_epsilon}
\end{table}

\begin{table}[!thb]
\centering
\begin{tabular}{@{}lcccccc@{}}
\toprule
 & \textit{Reorient} & \textit{Door} & \textit{Tool} & \textit{Tool\_short} & \textit{Reloc} & \textit{Reloc\_with\_failures}\\ \toprule
\textit{\LS{LS}} & \LS{100.0\%} & \LS{100.0\%} & \LS{0\%} & \LS{0\%} & \LS{100.0\%} & \LS{0.0\%}\\
\textit{SC($\varepsilon=0$)} & \textbf{100.0\%} & \textbf{100.0\%} & \textbf{100.0\%} & 99.5\% & \textbf{100.0\%} & \textbf{100.0\%} \\
\textit{SC($\varepsilon=10^{-5}$)} & \textbf{100.0\%} & \textbf{100.0\%} & \textbf{100.0\%} & 99.5\% & \textbf{100.0\%} & \textbf{100.0\%} \\
\textit{SC($\varepsilon=10^{-2}$)} & 99.2\% & \textbf{100.0\%} & \textbf{100.0\%} & \textbf{100.0\%} & \textbf{100.0\%} & \textbf{100.0\%}\\ \bottomrule
\end{tabular}
\vspace{0.3em}
\caption{Safety Rate for the learned policy with varying $\varepsilon$ values.}
\label{table:kodex_safe_epsilon}
\end{table}

\textbf{\textit{Franka Panda} dataset}: Though all variants of \textit{SC} perform similarly with sufficient data, we observe that $SC(\varepsilon=10^{-2})$ is the most sample-efficient and comparable to \textit{CG} and \textit{SOC} (see Fig.~\ref{fig:franka_pos_sc} and Table.~\ref{tab_distance_eps}). This is likely because the true spectral radius of the underlying system is likely closer to $0.99$.
\begin{figure*}[!htb]
    \centering
    \includegraphics[width=0.8\textwidth]{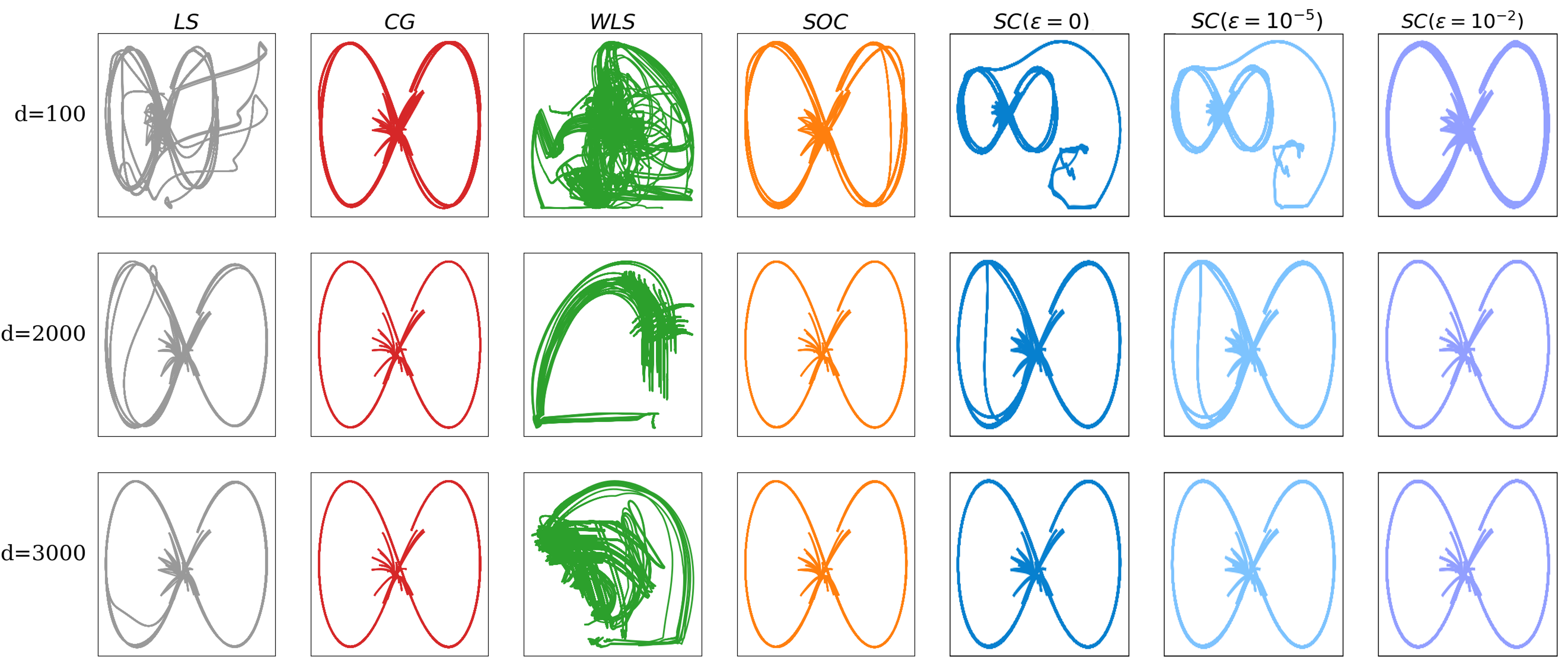}
    \caption{Rollouts of the end effector trajectory.}
    \label{fig:franka_pos_sc}
\end{figure*}

\begin{table}[!htb]
\centering
 \resizebox{\linewidth}{!}{
\begin{tabular}{@{}lccccc@{}}
\toprule
 & \textit{LS} & \textit{SC($\varepsilon=0$)} & \textit{SC($\varepsilon=10^{-5}$)} & \textit{SC($\varepsilon=10^{-2}$)}
 \\ \toprule
$d=100$ & $0.0748 \pm 0.1218$  & $0.0693 \pm 0.1370$ & $0.0686 \pm 0.1344$ & $0.0193 \pm 0.0074$\\
$d=2000$ & $0.0207 \pm 0.0173$ & $0.0183 \pm 0.0138$ & $0.0182 \pm 0.0132$ & $0.0175 \pm 0.0062$\\
$d=3000$ & $0.0164 \pm 0.0083$ & $0.0156 \pm 0.0057$ & $0.0156 \pm 0.0057$ & $0.0176 \pm 0.0062$\\
 \bottomrule
\end{tabular}
}
\vspace{0.3em}
\caption{Average prediction error of the end-effector's position (unit: $m$).}
\label{tab_distance_eps}
\end{table}

\section{Connections between \textit{SC} and Koopman modes}
\label{sec:koopman_modes}
We have also discovered interesting connections between \textit{SC} and Koopman modes, as mentioned in Section~\ref{sec:corrupted}, which provides us with valuable insights on how \textit{SC} works on corrupted datasets. For each eigenvalue $\lambda_i\ (i=1,...,n)$ of Koopman matrix $K$, we can find corresponding eigenvector $\mathbf{v_i}$ (i.e., the mode) and eigenvector $\mathbf{w_i}$ of adjoint matrix $K^*$, such that $K\mathbf{v_i}=\lambda_i\mathbf{v_i} \text{, and }
K^*\mathbf{w_i}=\overline{\lambda_i}\mathbf{w_i}$.
Then following \citep{rowley2009spectral}, we construct the corresponding eigenfunction $\varphi_i:\bold{R}^n\rightarrow C$, defined as $\varphi_i(\mathbf{z}) = \langle \mathbf{z}, \mathbf{w}_i \rangle$.
By doing so, the rollout of eigenfunctions can be represented by 
$\varphi_i(K\mathbf{z})=\lambda_i\varphi_i(\mathbf{z})$, and the states can be reconstructed by $\mathbf{z_k}=\sum_{i=1}^{n}\lambda_i^k\varphi_i(\mathbf{z_0})\mathbf{v_i}$.
Based on this formula, specific modes can be selected for rollouts.
\begin{figure*}[!thb]
    \centering \includegraphics[width=\textwidth, trim={0 0 0 0}]{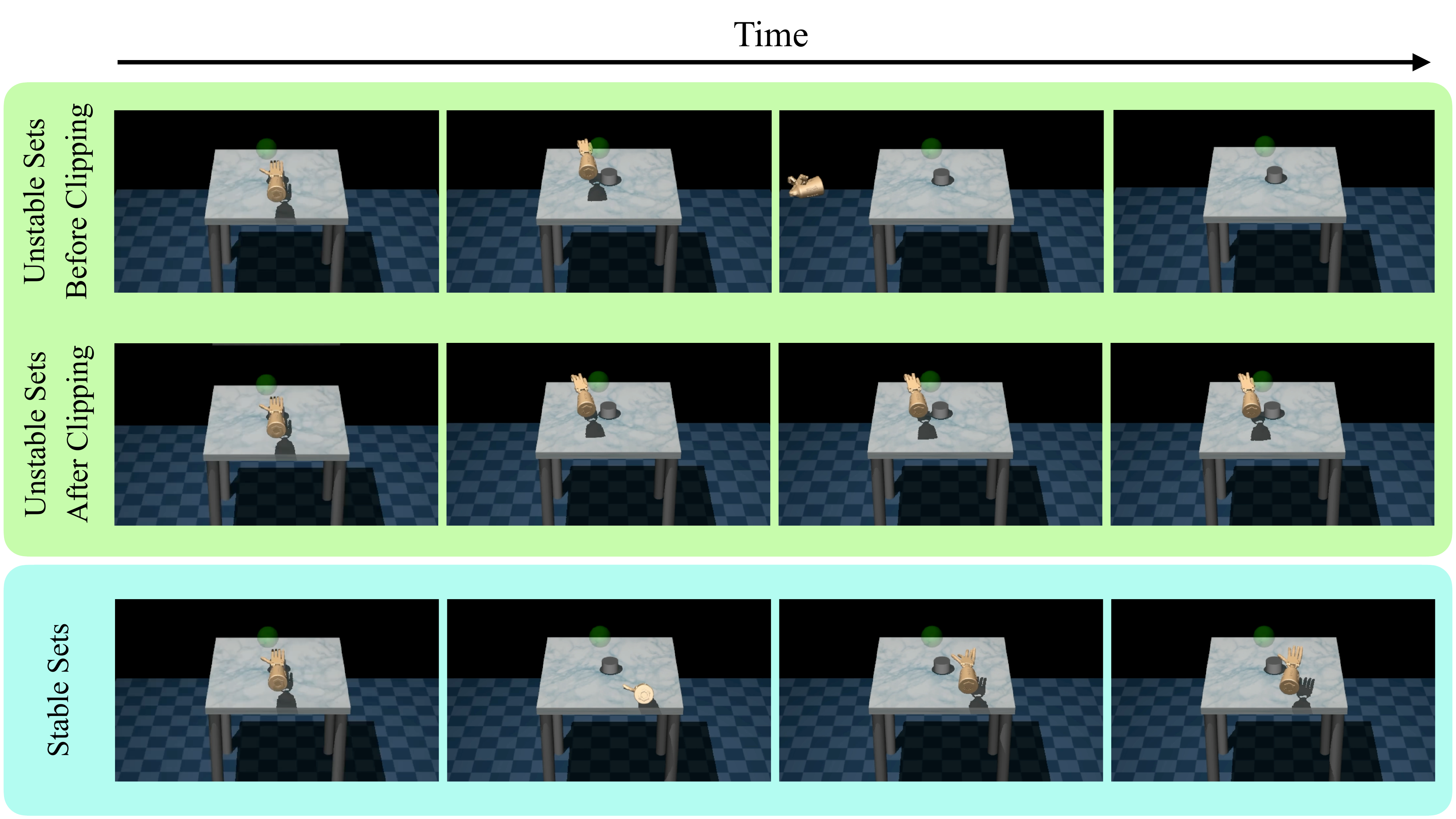}
\caption{Visualizations of rollout with two subsets of the Koopman modes. The green box represents rollouts using \textit{Unstable Set} (norm of eigenvalue larger than 1). The blue box represents rollout with the rest \textit{Stable Set}. In the green box, the upper line shows the rollout before clipping and the lower line shows it after clipping. It is obvious that after clipping, the rollout becomes stable, thereby allowing the composite rollout of both sets to safely complete the task (as shown in Table~\ref{table:kodex_task_hard}. The full policy rollout is available in the supplementary video.}
\label{fig:mode}
\end{figure*}
In our experiments, we first computed the Koopman matrix $K_f$ from $Reloc\_with\_failures$ dataset. We then sorted the eigenvalues by their norms from largest to smallest and divided them into two subsets for rollouts: (i) \textit{Unstable Set}: the first 52 Koopman modes with unstable eigenvalues, which have norms greater than 1, and (ii) \textit{Stable Set}: the remaining 707 modes. We then visualized the robot hand motions of each subset via the eigenfunction rollouts in Fig.~\ref{fig:mode}, respectively (see the supplementary video for full policy rollouts). It can be seen that the rollout of the \textit{Unstable Set} before clipping clearly generate unstable hand motions (see upper line of the green box), which results in the low success rate and safety rate when combined with \textit{Stable Set} for the final task performance (as shown in Table.~\ref{table:kodex_task_hard}). 
Further, we clipped the matrix $K_f$ and visualized the rollouts of the clipped \textit{Unstable Set} in Fig.~\ref{fig:mode}. It is obvious that after clipping, the robot hand motions become stable (see bottom line of the green box), enabling the composite system to safely complete the task (as also shown in Table~\ref{table:kodex_task_hard}). The full policy rollout is also available in the supplementary video.

\section{\textit{SC} can help learn stable LDS with control inputs}
\label{sec:control}

We finally evaluated \textit{SC}'s ability to learn stable LDSs with control inputs (see Appendix.~\ref{sec:SC_control} for ). 
To this end, we used \textit{Franka Panda}~\citep{gaz2019dynamic}, a dataset of Franka Panda Robot Arm's movements. The robot state space has 17 dimensions ($r=17$), including the coordinates of end effector ($\mathbb{R}^{3}$), joint angle values ($\mathbb{R}^{7}$), and joint angular velocities ($\mathbb{R}^{7}$), and the control input has $\mathbb{R}^{7}$ joint torques. This dataset consists of 3100 time steps of (randomly generated) control inputs and the corresponding robot states. 

To examine how dataset size affects performance, we adjust the dataset size to $d = 100, 2000$ and $3000$ samples. For each size, we randomly select five subsets from the whole dataset, resulting in totally 15 subsets. We then train LDSs, i.e., the system matrix $A$ and the control matrix $B$, using different approaches on all 15 subsets. 



We first evaluate the long-horizon prediction accuracy of each approach. For this, we first command the random control inputs in simulation with ground-truth system dynamics to generate test trajectories. We then rollout the trained $A$ and $B$ matrices with the same control inputs, and then compute the average $\mathscr{l}_2$ distance between the predicted and ground-truth end-effector positions for each method. From the results shown in Table.~\ref{tab_distance}, we observe that while \textit{SC} performs worse than \textit{CG} and \textit{SOC} when $d=100$, its performance becomes comparable as $d$ increases. Notably, \textit{SC} achieves this with significantly lower computation time and memory usage.

Additionally, we use the trained $A$ and $B$ matrices to construct an LQR controller, and use such controller to track a figure-8 pattern in simulation (similar in~\cite{mamakoukas2020memory}). How accurately the figure is tracked indicates how well the learned LDS aligns with the actual dynamical system with control inputs. For each approach, we repeat this process five times with different initial positions of the end effector. We show the qualitative tracking performances in Fig.~\ref{fig:franka_pos_sc} of Appendix.~\ref{appendix:stability-expresivity tradeoff}, where \textit{SC} demonstrates performance comparable to that of other methods.

\begin{table}[!thb]
\centering
 \resizebox{0.75\linewidth}{!}{
\begin{tabular}{@{}lcccccc@{}}
\toprule
 Dataset size & \textit{LS} & \textit{CG} & \textit{WLS} & \textit{SOC} & \textit{SC}
 \\ \toprule
$d=100$ & $0.0748 \pm 0.1218$ & $0.0166 \pm 0.0066$ & $0.3390 \pm 0.2059$ & $0.0225 \pm 0.0131$ & $0.0693 \pm 0.1370$\\
$d=2000$ & $0.0207 \pm 0.0173$ & $0.0154 \pm 0.0058$ & $0.3051 \pm 0.1853$ & $0.0154 \pm 0.0058$ & $0.0183 \pm 0.0138$\\
$d=3000$ & $0.0164 \pm 0.0083$ & $0.0155 \pm 0.0058$ & $0.2845 \pm 0.1501$ & $0.0155 \pm 0.0058$ & $0.0156 \pm 0.0057$\\
 \bottomrule
\end{tabular}
}
\captionsetup{skip=2pt}
\caption{Average prediction error of the end-effector's position (unit: $m$).}
\label{tab_distance}
\end{table}

\section{Details of \textit{DexManip} tasks}\label{sec:kodex}
\textit{DexManip} is a dataset collected when using Adroit Hand to conduct four different dexterous manipulation tasks in MUJOCO simulation environment. In this dataset, the states of hand and states of object is captured, concatenated and lifted to a higher-dimension Hilbert space at each time step (see~\cite{han2023KODex}), so that the hand-object system becomes an LDS in the Hilbert space. Thereafter, we use a simple PD controller to track the generated hand trajectories in the simulator. The tasks are detailed as follows.

\noindent\textit{Pen Reorientation:} This task is mainly about orienting a pen to a random target rotation angle. The hand and object states are lifted to a higher-dimensional $z$ states and different approaches are used to learn the LDSs in the $z$ state space. This dataset consists of 131 rollouts that pens are reoriented to the target rotation angles successfully ($T=100$). At each time step $t$, if $\mathrm{o}^{\text{goal}} \cdot \mathrm{o}^{\text{pen}}_t > 0.90$ ($\mathrm{o}^{\text{goal}} \cdot \mathrm{o}^{\text{pen}}_t$ measures orientation similarity), then we have $\rho(t) = 1$. This task is considered i) successful if $\sum_{t=1}^T \rho(t) > 10$, and ii) safe if the pen remains in the robot hand at the final time step (i.e., it is not dropped).

\noindent\textit{Door Opening:} This task is mainly about opening a door from a random starting position. The hand and object states are lifted to a higher-dimensional $z$ states and different approaches are used to learn the LDSs in the $z$ state space. This dataset contains 158 rollouts in which the doors are fully opened ($T=70$). The task is considered i) successful if at last time step $T$, the door opening angle is larger than 1.35 rad, and ii) safe if the robot hand is still near the door (i.e., the distance is smaller than 0.5m) at final time step, indicating that the robot hand does not drifts away. 

\noindent\textit{Tool Use:} This task is mainly about picking up the hammer to drive the nail into the board placed at some height. The hand and object states are lifted to a higher-dimensional $z$ states and different approaches are used to learn the LDSs in the $z$ state space. There are two types of datasets:
\begin{itemize}       
    \item \textit{Tool}: This dataset consists of 200 rollouts that continues after the nail is driven into the board, so that the hammer stays stable in the air ($T=100$). This is the original setting proposed in \cite{Rajeswaran-RSS-18}.
    \item \textit{Tool\_short}: This dataset consists of 200 rollouts that end right after the nail is driven in ($T=35$).
\end{itemize}


This task is considered i) successful if at last time step (note that during evaluation, the policy will always run for 100 time steps, similar to the \textit{Tool} dataset), the Euclidean distance between the final nail position and the goal nail position is smaller than 0.01m, and ii) safe if the final distance from the hammer to the goal nail position is smaller than 0.3m, indicating that the robot hand neither drifts away nor throws away the hammer.

\noindent\textit{Object Relocation:} This task is mainly about moving a cylinder to a randomized target location. Similarly, all approaches are operated in the higher-dimensional $z$ state space. The two types of datasets are:

\begin{itemize}
    \item \textit{Reloc}: This dataset consists of 177 rollouts that objects are all moved to the target position successfully in the end ($T=100$).
    \item \textit{Reloc\_with\_failures}: In addition to the \textit{Reloc} dataset, this one also includes the failure cases where the objects are failed to move to the target position in the end ($T=100$).
\end{itemize}

At each time step, if the Euclidean distance between the current cylinder position and the target position is smaller than 0.10m, then we have $\rho(t) = 1$. This task is considered i) successful if $\sum_{t=1}^{T=100} \rho(t) > 10$, and ii) safe if the robot hand is still near the target position (i.e., the distance is smaller than 0.5m) at final time step, indicating that the robot hand does not drifts away.

\end{document}